\documentclass[accepted]{uai2026} % for initial submission

\usepackage[american]{babel}
\usepackage[disable]{todonotes}
\usepackage{natbib} % has a nice set of citation styles and commands
\usepackage{mathtools} % amsmath with fixes and additions
\usepackage{booktabs} % commands to create good-looking tables
\usepackage{tikz} % nice language for creating drawings and diagrams
\usepackage{amssymb}

\usepackage{xfrac}
\usepackage{tcolorbox}
\usepackage{amsthm}

\theoremstyle{definition}
\newtheorem{thm}{Theorem}[section]

\newtheorem{prop}[thm]{Proposition}

\title{Task Expansion and Cross Refinement for Open-World Conditional Modeling}

\author[1]{\href{mailto:<shbhat@cs.unc.edu>?Subject=Your 2026 paper}{Shreyas Bhat Brahmavar}{}}
\author[1]{Qiyang Liu}
\author[2]{Yang Li}
\author[1]{Junier Oliva}

\affil[1]{%
    Department of Computer Science\\
    University of North Carolina at Chapel Hill
}
\affil[2]{%
    Independent Researcher
}
\begin{document}
\maketitle

\begin{abstract}
Open-world conditional modeling (OCM), requires a single model to answer arbitrary conditional queries across heterogeneous datasets, where observed variables and targets vary and arise from a vast open-ended task universe. Because any finite collection of real-world datasets covers only a small fraction of this space, we propose Task Expansion and Cross Refinement (TEXR), a semi-supervised framework that enlarges effective task coverage through structured synthesis and refinement of semantic data contexts. TEXR first generates diverse uninstantiated dataset schemas and weakly instantiates them via structured probabilistic generators guided by large language models. It then performs cross-model refinement by training on disjoint data partitions and revising synthetic values across splits to reduce confirmation bias and improve pseudo-value quality. The refined synthetic datasets are aggregated with real data to train a unified conditional model. Across heterogeneous tabular benchmarks, TEXR consistently improves zero-, few-, and many-shot performance for multiple OCM backbones, demonstrating that structured task expansion and cross refinement enhance open-world conditional modeling.\looseness-1
\end{abstract}

\section{Introduction}

\begin{figure}[t]
  \centering
  \includegraphics[width=\linewidth]{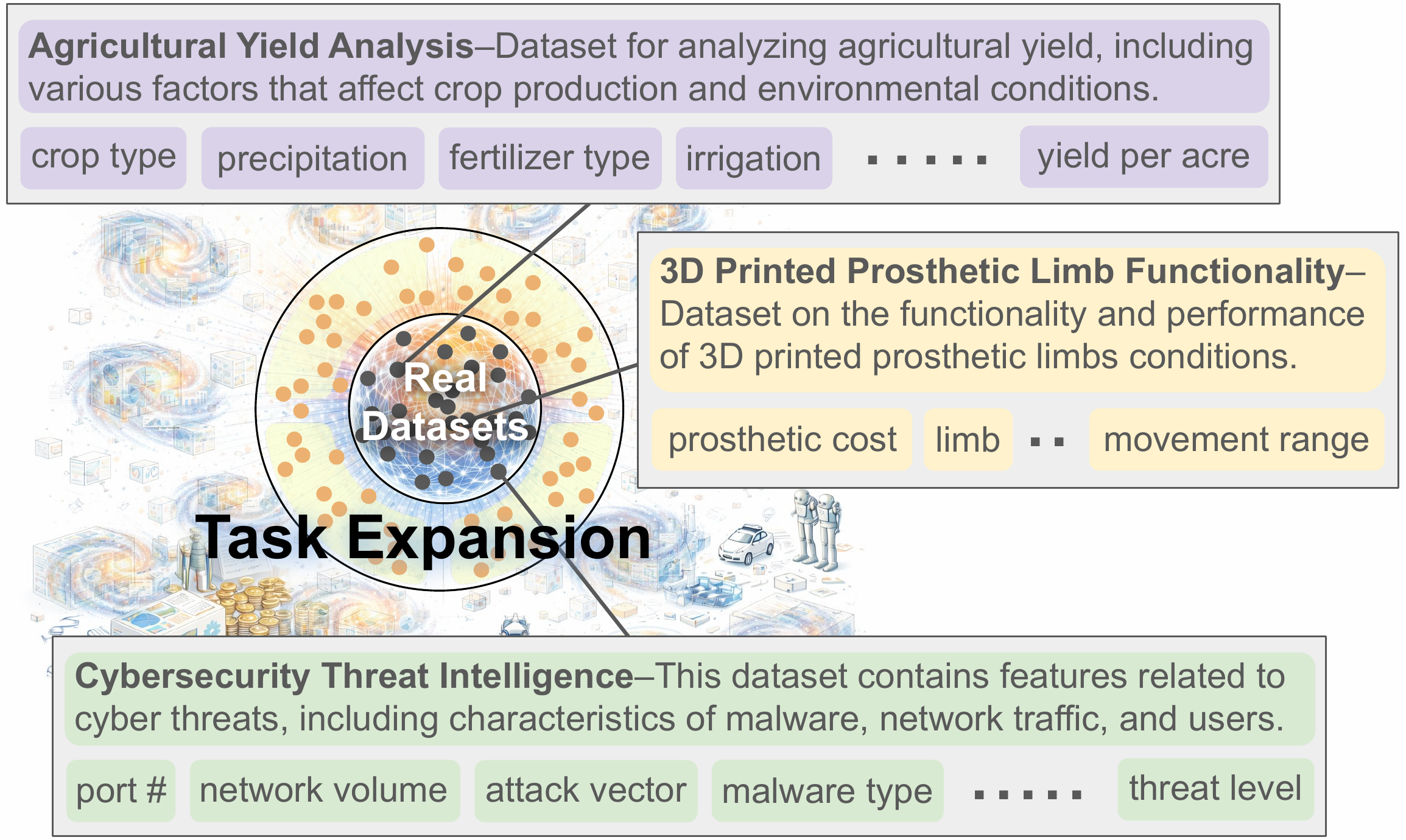}
  \caption{The \textbf{task universe} for open-world conditional modeling is vast, and any finite collection of real-world datasets (gray) captures only a small fraction of it. We propose synthetic data generation and refinement techniques to expand task coverage (orange).}
  \label{fig:universe}
\vspace{-5pt}
\end{figure}

Modern foundation models—most prominently large language models (LLMs)—have demonstrated an emerging ability to perform inference across a strikingly broad range of tasks without task-specific retraining, suggesting a pathway toward general-purpose reasoning systems \citep{brown2020languagemodelsfewshotlearners, wei2022emergentabilitieslargelanguage}. Inspired by this paradigm, a recent subset of tabular foundation models such as TP-BERTa \citep{yan2024makingpretrainedlanguagemodels}, CM2 \citep{ye2024crosstablemaskedpretrainingweb}, GTL \citep{wen2024supervisedgenerativenovelparadigm}, TabSTAR \citep{arazi2025tabstar} and ASPIRE \citep{brahmavar2026universalneurallikelihoodinference} aim to perform \textbf{semantically-grounded open-world conditional modeling} (OCM)\footnote{Some tabular foundational models do not perform semantically grounded OCM and instead provide a finetuneable or in-context adjusted pretrained model for a new \emph{many-shot} dataset.}. Given \emph{little} (few-shot) or \emph{no} (zero-shot) examples, along with an arbitrary set of observed variables and their semantic descriptions, these models estimate a designated target—even when the feature set or prediction task was never explicitly encountered during training. 

This capability is crucial for enabling statistically informed predictions in \textbf{data-starved or hypothetical settings}, where labeled data for the exact task may be scarce or unavailable. For example, a single generalist model might be asked to predict: (i) the risk of septic shock from a sparse subset of labs and vitals in a newly integrated hospital system; (ii) regional hospitalization burdens from partial demographic and policy indicators; or (iii) gene-expression responses from a limited panel of molecular measurements. 

However, this \emph{open-world inference problem} is exceptionally difficult because the support of possible tasks spans the combinatorial space of all feature subsets, target choices, and domain-specific data-generating processes. Although there has been a growing effort to aggregate large collections of real-world datasets for pretraining, even expansive collections (spanning thousands of datasets) represent only a \textbf{small fraction of this universal task support}. Consequently, models are forced to operate largely in out-of-distribution regimes at test time. Since acquiring real-world data across vast encompassing scenarios is prohibitive, this inherent coverage gap remains a fundamental barrier to open-world conditional modeling (see Fig.~\ref{fig:universe}).\looseness-1

To bridge this gap, we introduce the first principled, semi-supervised framework designed specifically to scale task coverage for OCM without requiring additional real-world data. We begin by recasting OCM as a \emph{context–to–feature} generative problem. Specifically, \emph{we show an equivalence} between training a model to map arbitrary features to arbitrary targets, and training a conditional generator that produces coherent feature values conditioned purely on a semantic data context (table-level and column-level descriptions, variable names, and types). This reframing naturally unlocks a semi-supervised formulation. In addition to fully observed datasets (semantic contexts paired with instance-level values), we can incorporate \emph{context-only specifications}—collections of dataset/variable descriptions and metadata without any associated data rows. These "uninstantiated" contexts define the structure of a task without providing realizations, functioning analogously to unlabeled examples in classical semi-supervised learning.\looseness-1

To construct and leverage the ``unlabeled'' data (i.e., ``uninstantiated'' synthetic contexts with no feature values), we introduce \textbf{Task Expansion and Cross Refinement (TEXR)}, a semi-supervised framework designed to scale effective task coverage. TEXR first performs \emph{task expansion} by synthesizing diverse, semantically grounded data contexts (schemas with feature and dataset descriptions) and \emph{weakly instantiates} them using structured probabilistic generators guided by large language models. This produces heterogeneous synthetic datasets that represent plausible but previously unseen conditional tasks. It then applies \emph{cross refinement}, where models are trained on disjoint data partitions and used to revise each other’s synthetic labels across splits. Finally, the refined synthetic datasets are aggregated with real data to train a unified, open-world conditional model.

\textbf{Contributions}\quad Our core contributions are as follows:
\begin{itemize}[leftmargin=*,topsep=-6pt, noitemsep, wide]
\item \textbf{Semi-Supervised Learning for OCM.} We establish a principled mathematical link between arbitrary open-world prediction and context-conditioned generation, unlocking a semi-supervised pathway for tabular modeling.
\item \textbf{Synthetic Task Augmentation.} We develop an approach to scale task coverage by synthesizing diverse, semantically grounded contexts and weakly instantiating them via LLM-guided structured probabilistic generators.
\item \textbf{Cross-Model Refinement.} We propose a procedure where models trained on disjoint data partitions  
revise synthetic pseudo-instances to refine their values while reducing error amplification and mitigating confirmation bias.\looseness-1 %revise synthetic values to  effectively denoising pseudo-labels while mitigating confirmation bias.
\item \textbf{Empirical Experiments.} Across heterogeneous tabular benchmarks, TEXR yields considerable performance gains in zero-, few-, and many-shot regimes for multiple OCM backbones, outperforming existing state-of-the-art synthetic data and pseudo-labeling baselines.
\item \textbf{A Public Corpus of Synthetic Tasks.} Alongside our code, we release a collection of 10{,}000+ high-quality, refined synthetic datasets spanning a broad range of scenarios to accelerate future research in open-world tabular inference.
\end{itemize}

\section{Related Works}\label{sec:related_work}

\subsection{Tabular Foundation Models}

%Recent work has shifted from per-dataset training toward models (pre)trained across large collections of heterogeneous tables to enable cross-dataset generalization.

\textbf{Transfer Learning for Many-shot Inference.}
A central objective in tabular learning is to leverage patterns learned from diverse collections of datasets to improve performance on a new target dataset unseen during pretraining. Unlike conventional per-dataset models (e.g., MLPs or gradient-boosted trees), these approaches aim to transfer statistical structure across tasks.

The TabPFN family \citep{hollmann2023tabpfntransformersolvessmall, ye2025closerlooktabpfnv2, grinsztajn2026tabpfn25advancingstateart}, TabICL \citep{qu2025tabicltabularfoundationmodel}, and LimiX \citep{zhang2025limixunleashingstructureddatamodeling} are strong cross-tabular predictors based on in-context or meta-learning paradigms. These models primarily rely on learned column embeddings and summary-statistic representations derived from feature values, and generally do not incorporate rich natural-language descriptions of datasets or features.

\textbf{Semantically-Grounded Tabular Learning.}
In parallel, another line of work incorporates textual and metadata descriptions to ground tabular prediction semantically over arbitrary feature–target configurations. This is particularly important in zero- or (very) few-shot OCM settings, where summary statistics alone are insufficient (or unavailable in the zero-shot case) to characterize a new dataset.

CM2 \cite{ye2024crosstablemaskedpretrainingweb} and TP-BERTa \cite{yan2024makingpretrainedlanguagemodels} pretrain shared backbones via cross-table objectives and finetune for downstream tasks. CM2 employs masked table modeling that integrates feature semantics, while TP-BERTa adapts a pretrained language model by fusing textual feature descriptions with discretized numerical values. ASPIRE \cite{brahmavar2026universalneurallikelihoodinference} extends cross-table masked modeling with semantic conditioning, training transformer architectures that aggregate heterogeneous features while leveraging dataset- and column-level textual context. TabSTAR \cite{arazi2025tabstar} considers semantically target-aware representations by unfreezing a pretrained text encoder and injecting target tokens as inputs, enabling cross-dataset transfer without dataset-specific output heads. LLM-based approaches such as TabLLM \cite{hegselmann2023tabllmfewshotclassificationtabular}, FeatLLM \cite{han2024largelanguagemodelsautomatically}, and GTL \cite{wen2024supervisedgenerativenovelparadigm} serialize rows into text and cast prediction as conditional language modeling via prompting or supervised fine-tuning. These methods leverage in-context learning and semantic reasoning to generate predictions through next-token likelihood conditioned on feature values and textual context.

\emph{Motivated by low-data OCM, we study how TEXR enhances these semantically grounded cross-tabular models.}

\subsection{Synthetic Dataset Generation for Cross-Tabular Generalization}

\textbf{Synthetic Priors for Cross-Tabular Transfer.}
TabPFN-style models \citep{hollmann2023tabpfntransformersolvessmall, ye2025closerlooktabpfnv2, grinsztajn2026tabpfn25advancingstateart} demonstrate that training on large collections of synthetic datasets can induce strong cross-dataset generalization. These models sample tasks from structured priors over functional classes, feature dependencies, noise, and sample regimes, enabling in-context transfer at test time. However, their synthetic data operate over abstract, black-box feature spaces and do not incorporate dataset-level semantics, column descriptions, or natural-language context.\looseness-1

\textbf{Semantically-Grounded Synthetic Generation.}
Recent work leverages semantic metadata and large language models to generate faithful tabular samples, particularly in low-data settings. Approaches such as StructSynth \citep{liu2025structsynthleveragingllmsstructureaware}, CTSyn \citep{lin2025ctsynfoundationmodelcross}, GReaT \citep{borisov2023languagemodelsrealistictabular}, TabuLa \citep{zhao2025tabulaharnessinglanguagemodels}, and CLLM \cite{seedat2024curatedllmsynergyllms} condition generation on textual descriptions or schema information \emph{along with a few instances} to produce additional samples. 
% These methods are typically used and evaluated in single-dataset settings, where synthetic data is judged on fidelity to that dataset or augments/regularizes a model trained on that dataset, rather than broadening the task distribution encountered during cross-dataset pretraining on various OCM backbones. \todojo{better phrasing?}
% These methods are typically used and evaluated in single-dataset settings, where synthetic data is assessed for within-dataset fidelity or used to augment and regularize a model trained on that dataset, rather than to expand the task distribution underlying cross-dataset pretraining of OCM models.
% These methods are typically used and evaluated w.r.t.~single-test-time-dataset settings, where synthetic data is assessed for within-dataset fidelity to an unseen dataset or used to augment and regularize a model trained on that dataset, rather than to expand the task distribution underlying cross-dataset pretraining of OCM models.
These methods are typically used and evaluated w.r.t.~a single target dataset at test time, where synthetic data is assessed for fidelity to that dataset or used to augment and regularize its downstream model, rather than to expand the task distribution underlying cross-dataset 
%OCM pretraining.
pretraining for OCM backbones

% \textbf{TEXR: Expanding the Semantic Task Universe.}
% \emph{TEXR instead uses synthetic data to (starting from a zero-shot dataset free approach) enlarge effective task coverage for semantically-grounded cross-tabular models that enables semi-supervised regime that improves OCM.} 
\emph{TEXR instead uses synthetic data—starting from a zero-shot, dataset-free formulation—to expand effective task coverage for semantically grounded cross-tabular models, enabling a semi-supervised regime that improves OCM.}

\subsection{Pseudo-Labeling For Tabular Foundational Models}

% Pseudo-labeling and self-training originate in semi-supervised learning, beginning with the general pseudo-labeling framework of \cite{lee2013}, which treats high-confidence predictions on unlabeled data as training targets, and subsequent refinements such as confidence-based pseudo-labeling (CPL) \cite{TODO}, which filter or weight pseudo-labels to mitigate confirmation bias. In tabular learning, these methods are primarily applied within single-dataset settings, focusing on improving pseudo-label reliability rather than enhancing open-world cross-dataset generalization. For example, CAST \cite{kim2023cast} calibrates pseudo-label confidence using density- and cluster-assumption signals to reduce noise during self-training on a specific dataset \todojo{Is this true?}. Closer to cross-dataset pseudo-labeling, MediTab \cite{ijcai2024p670} augments across medical tables by generating pseudo-labels to impute unobserved patient characteristics, leveraging the fact that patients share a coherent universe of attributes across datasets; its learn--annotate--audit pipeline refines supplementary data before retraining. However, this strategy relies on domain-specific feature universality (e.g., patient attributes) and does not directly address open-world settings in which feature universes differ across datasets.

Pseudo-labeling is a standard semi-supervised strategy for neural networks, originating with \citep{lee2013pseudo}, which treats high-confidence predictions on unlabeled data as training targets.
%, with later refinements such as confidence-based pseudo-labeling (CPL) that filter or weight pseudo-labels to reduce confirmation bias. 
In tabular learning, these approaches are primarily applied within single-dataset settings to improve pseudo-label reliability rather than cross-dataset open-world generalization \citep{yoon2020vime, farahani2022clpl, kim2023revisiting}. For example, CAST \citep{kim2023cast} calibrates pseudo-label confidence using density- and cluster-based signals during self-training on a given table. STUNT \citep{nam2023stunt} augments a dataset by appending cluster-derived pseudo-labels to enhance representation learning, rather than expanding the set of datasets into new task scenarios. MediTab \citep{wang2024meditab} leverages cross-dataset medical patterns to generate synthetic patients by imputing shared patient characteristics into existing tables. However, MediTab relies on the shared applicability of patient attributes across datasets and does not extend beyond the clinical attribute space into new semantic task scenarios.

% \emph{In contrast, TEXR performs instance-level value refinement across heterogeneous semantic contexts, using synthetic data to expand effective task coverage and improve zero-/few-shot OCM }
% Unlike prior approaches that refine pseudo-labels within a fixed dataset or domain, 

\emph{In contrast, TEXR performs instance-level value refinement to instantiate new structured tasks from semantic contexts, expanding the task universe for zero-/few-shot OCM.}

\section{Method}\label{sec:method}
We derive a semi-supervised framework that treats semantic data contexts as the fundamental unit of unsupervised learning for open-world conditional modeling. To address the coverage gap of real-world datasets (Fig.~\ref{fig:universe}), our framework—\textbf{Task Expansion and Cross Refinement (TEXR)}—constructs a diverse pool of uninstantiated contexts, generates weak but structured synthetic tabular values, and iteratively refines them through a cross-model co-training procedure. This pipeline expands the model's coverage over the open-world task distribution while maintaining alignment with real-world data priors.

\begin{figure}[t]
  \centering
  \includegraphics[width=.99\linewidth]{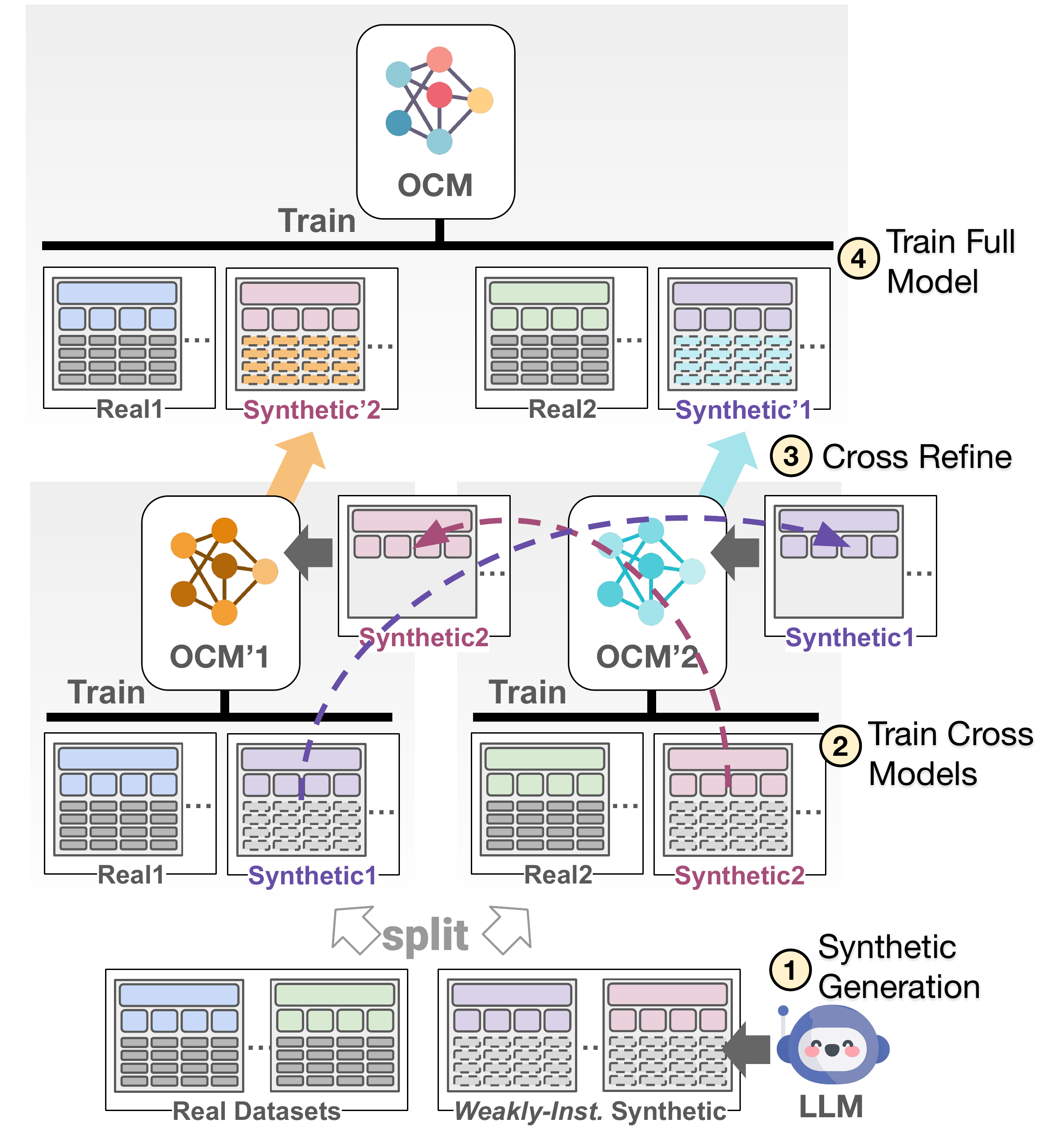}
  \caption{\textbf{TEXR} provides an initial collection of weakly-instantiated datasets that are partitioned and then refined with cross model labeling (which refine values on synthetic data not seen during training). Finally, revised synthetic and real datasets are consolidated to train a full open-world conditioning model.\label{fig:texr}}
  
\end{figure}

\subsection{OCM as Context Conditioned Joint Modeling}
As aforementioned, recent efforts have resulted in various models capable of performing open-world conditional modeling. Formally, these models aim to learn a parameterized conditional estimate of $p_\theta(x_t \mid x_o, \mathbf{S})$; at inference time, given a semantic data context $\mathbf{S}$ (e.g., dataset metadata and natural language feature descriptions) and an arbitrary set of observed feature values $x_o$\footnote{For notational simplicity, we formulate the zero-shot case here; the few-shot case is analogous but additionally incorporates a support set of examples.}, the objective is to estimate a designated target variable $x_t$.

To achieve arbitrary conditioning and arbitrary target prediction, open-world conditional models are typically trained based on minimizing the negative log conditional likelihood over a vast task distribution:
\begin{equation}\label{eq:ocm_loss}
-\mathbb{E}_{\mathbf{S}\sim\mathcal{T}}\mathbb{E}_{x\sim\mathcal{P}_{\mathbf{S}}}\mathbb{E}_{o,t\sim\mathcal{M}_{\mathbf{S}}}[\log p_{\theta}(x_{t} \mid x_{o},\mathbf{S})]
\end{equation}
where $\mathcal{T}$ represents the universal task distribution of semantic contexts, $\mathcal{M}_{\mathbf{S}}$ is a masking distribution that randomly selects observed features $o$ and target $t$, and $\mathcal{P}_{\mathbf{S}}$ is the true data distribution from which actual instances $x$ are sampled.

However, this native formulation is not directly conducive for synthetic data generation and semi-supervised learning since the generation of "unlabeled" examples is \emph{no easier} than the generation of "labeled" examples. That is, the generation of an unlabeled example--a coherent collection of feature descriptions, feature values, and dataset description--is itself a "labeled" example, since we can transfer over any one of the features to act as a target. Below, we mathematically \emph{decouple} the task semantic data descriptions from the observed data values for a formulation that more readily leads to a semi-supervised framework.

% However, this native formulation exposes a critical limitation: training relies entirely on fully observed data, since the masking expectation $\mathbb{E}_{o,t\sim\mathcal{M}_{\mathbf{S}}}$ must accommodate arbitrary feature and target subsets, the data expectation $\mathbb{E}_{x\sim\mathcal{P}_{\mathbf{S}}}$ inherently requires sampling from complete, fully populated data. Under this paradigm, scaling the model's open-world capabilities is fundamentally bottlenecked by the immense cost and practical difficulty of collecting coherent instance data across a vast diversity of semantic contexts.

% To bypass this requirement and mathematically decouple the task structure from the observed data, we propose an alternative formulation.

\begin{prop} \label{thm:OCM}
Open-world conditional model training on arbitrary inputs and targets is equivalent to training based on context-conditioned joint modeling. 
\end{prop}

\emph{Proof Sketch}.\quad
We note that arbitrary conditioning likelihoods \citep{strauss2021arbitrary} $\log p_\theta(x_t \mid x_o, \mathbf{S})$ are directly capable of modeling a joint likelihood via the chain rule of probability under an arbitrary feature permutation $\pi$: $\log p_\theta(x \mid \mathbf{S}) = \sum_{i=1}^{d_\mathbf{S}} \log p_\theta(x_{\pi_i} \mid x_{\pi_{<i}}, \mathbf{S})$. 

% Crucially, if the masking distribution $\mathcal{M}_\mathbf{S}$ in Eq.~\eqref{eq:ocm_loss} is defined by uniformly sampling a permutation $\pi$ and an index $i \in \{1, \dots, d_{\mathbf{S}}\}$ such that the observed features are $o = \pi_{<i}$ and the target is $t = \pi_i$, then the stochastic training objective \eqref{eq:ocm_loss} is mathematically equivalent to minimizing the normalized negative log joint likelihood:
When the masking distribution $\mathcal{M}_\mathbf{S}$ in Eq.~\eqref{eq:ocm_loss} uniformly produces targets and observed feature subsets (of arbitrary cardinality, as is common \cite{ye2024crosstablemaskedpretrainingweb, brahmavar2026universalneurallikelihoodinference}), then the stochastic training objective \eqref{eq:ocm_loss} is mathematically equivalent to minimizing the normalized negative log joint likelihood:
\begin{align}
    &-\mathbb{E}_{\mathbf{S}\sim\mathcal{T}}\mathbb{E}_{x \sim \mathcal{P}_\mathbf{S}} \left[ \frac{1}{d_\mathbf{S}} \log p_\theta(x \mid \mathbf{S}) \right] = \label{eq:uci_joint_loss}\\
    &-\mathbb{E}_{\mathbf{S}\sim\mathcal{T}}\mathbb{E}_{x \sim \mathcal{P}_\mathbf{S}} \mathbb{E}_\pi \left[ \frac{1}{d_\mathbf{S}} \sum_{i=1}^{d_\mathbf{S}} \log p_\theta(x_{\pi_i} \mid x_{\pi_{<i}}, \mathbf{S}) \right]. \label{eq:uci_semantic_loss}
\end{align}
That is, the standard OCM loss \eqref{eq:ocm_loss} stochastically subsamples the conditionals in the joint factorization \eqref{eq:uci_semantic_loss}. \qed

This generative reframing \eqref{eq:uci_joint_loss} clearly delineates the inputs to OC models as the \emph{semantic data contexts}, $\mathbf{S}$, and the outputs as the corresponding instance values, $x$. Crucially, this equivalence naturally unlocks a semi-supervised learning paradigm. By establishing OCM as a context-conditioned joint generator $p_\theta(x \mid \mathbf{S})$, we can formally treat \emph{uninstantiated semantic contexts} (dataset schemas/meta-data with no data rows) as our ``unlabeled'' instances. 

% Unlike traditional unlabeled data, which must be painstakingly collected from the real world, these semantic contexts can be synthetically generated at scale to cover an arbitrarily diverse universe of tasks. Synthesizing full data rows for these empty contexts then equates to generating structured pseudo-labels for the joint distribution. Furthermore, because the OCM itself learns this joint density, it is theoretically sound to use the trained model to iteratively evaluate and refine the statistical coherence of these synthetic values. This joint generative perspective directly motivates our TEXR pipeline, the details of which we introduce below.

\textbf{Semi-Supervised OCM}\quad
%Our framework \eqref{eq:uci_joint_loss}, establishes the OCM task as conditional modeling problem to generate values, $x$, given a corresponding semantic-grounding context $\mathbf{S}$. This 
Our framework \eqref{eq:uci_joint_loss} implies that we may view OCM as training on a \emph{supervised} collection of instantiated datasets: $\mathfrak{D}=\{(\mathbf{S}^{(k)}, \mathbf{X}^{(k)})\}_{k=1}^N$, where $\mathbf{X}^{(k)}$ is a dataset of values for $n^{(k)}$ instances, corresponding to $\mathbf{S}^{(k)}$. Therefore, an \emph{unsupervised} collection would then correspond to a set of  $\mathfrak{U}= \{ \tilde{\mathbf{S}}^{(k)}\}_{k=1}^M$ of semantic contexts without corresponding dataset values (i.e., uninstantiated semantic contexts). While this opens an avenue for semi-supervised learning, it begs the question of how one may collect an uninstantiated collection $\mathfrak{U}$. Unlike in classic settings like computer-vision, where unlabeled instances abound, the notion of an uninstantiated dataset is much rarer.
Here, we propose to generate a diverse collection of hypothetical data scenarios and respective contexts.
% Perhaps one may mine for proposed studies for data collection that describe data contexts and features to be recorded \cite{TODO}. 
% However, such examples are still unlikely to provide much coverage of the vast possible task universe. Instead, here we propose an alternative approach: generate a diverse collection of hypothetical data scenarios and respective contexts.

\subsection{Task Expansion via Synthetic Context}\label{sec:method:syn}

While uninstantiated semantic contexts provide a principled source of unsupervised data for OCM, acquiring them at the scale required to meaningfully expand task coverage is impractical from real-world sources alone. We therefore introduce \emph{Task Expansion}: a systematic synthesis of diverse dataset contexts using large language models (LLMs), leveraging their broad semantic priors.
We construct the unsupervised context pool through a hierarchical two-stage procedure (see Supp.~Mat.~$\S$\,\ref{sec:additional_results}):

\begin{enumerate}[leftmargin=*,topsep=0pt,itemsep=0pt]
    \item \emph{Topic Universe Construction.}  
    To promote broad domain coverage, we prompt an LLM to first enumerate high-level domains (e.g., \emph{Technology \& Digital}, \emph{Health \& Sciences}, \emph{Finance \& Business}), decompose them into sub-fields, and finally generate concrete predictive scenarios (e.g., \emph{hospital readmission risk}, \emph{customer churn}, \emph{loan default prediction}). This yields a large topic set $\mathfrak{T}=\{\tau_k\}_{k=1}^{K}$ spanning thousands of hypothetical tasks.

    \item \emph{Schema Generation.}  
    For each topic $\tau \in \mathfrak{T}$, we prompt the LLM to produce a structured (e.g., JSON) semantic context, consisting of: (i) a natural-language \emph{dataset description} summarizing the overarching task and data, and (ii) a feature list in which each variable is paired with a data type and a detailed semantic description.
\end{enumerate}

This procedure yields a large collection of uninstantiated semantic contexts,
$\mathfrak{U}=\{\tilde{\mathbf{S}}^{(k)}\}_{k=1}^{M}$, each containing rich dataset-/column-level descriptions but no associated rows. Thus, $\mathfrak{U}$ constitutes a purely unsupervised pool of candidate tasks derived from LLM priors.
In the following section, we develop a pseudo-labeling, cross-model refinement procedure that leverages $\mathfrak{U}$ for semi-supervised OCM training.

\begin{figure}[t]
  \centering
  \includegraphics[width=.95\linewidth]{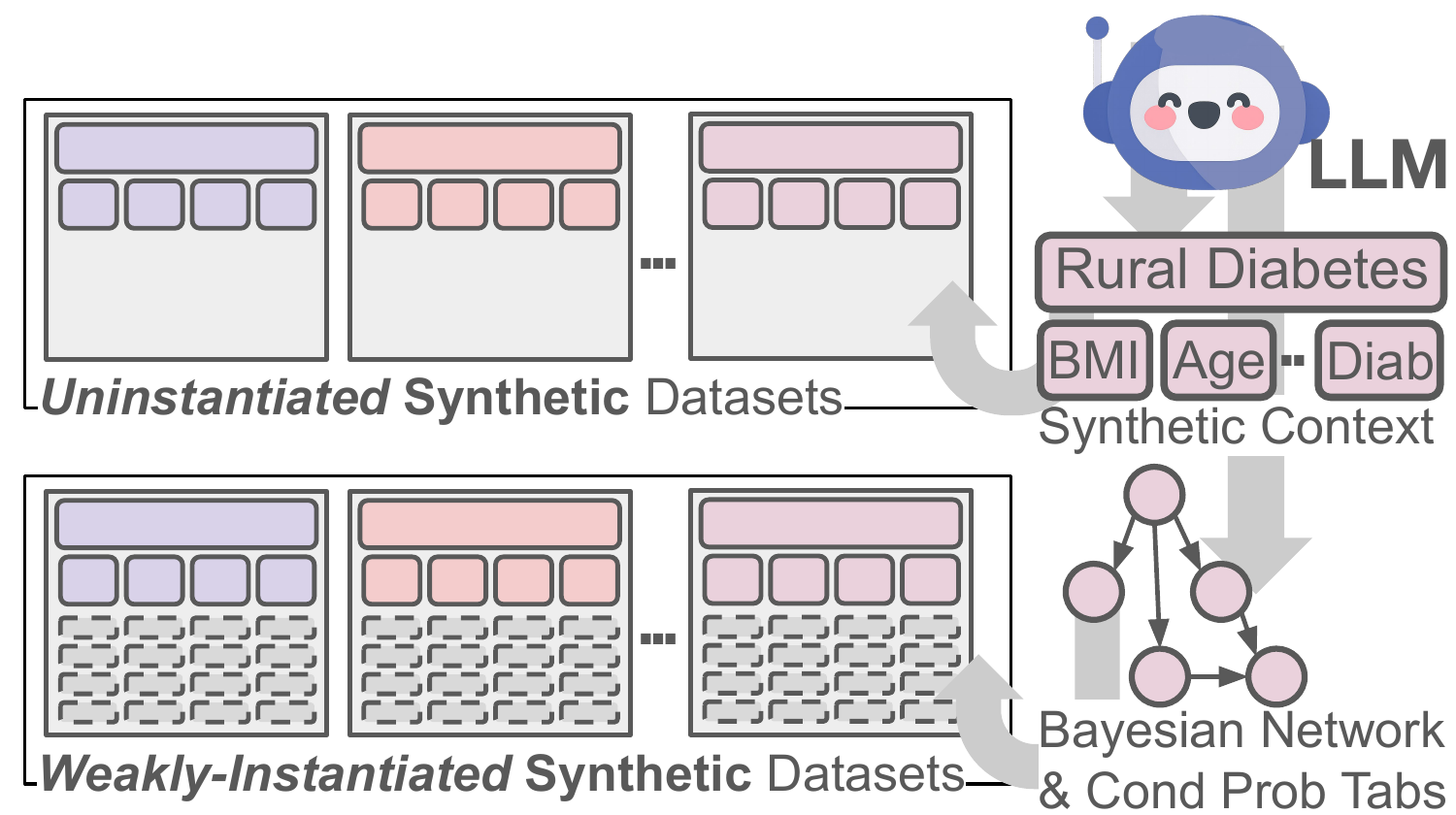}
  \caption{Our initial collection of \textbf{synthetic datasets} stems from an LLM guided process that first generates hypothetical data contexts followed by a weak-labeling scheme that generates corresponding Bayesian networks and conditional probability tables to produce instance values. }
\end{figure}

\subsection{Pseudo-Value Generation and Cross Refinement}

For OCM training that leverages the unsupervised pool of synthetic contexts, $\mathfrak{U}$, we introduce a semi-supervised procedure that sequentially couples weak data instantiation with a partition-induced cross-refinement strategy. This phase fills (instantiates) our empty schemas with structured initial values and systematically revises them.

\subsubsection{Weak Instantiation and Filtration} \label{sec:method:weak}

We begin by assigning preliminary values to the uninstantiated semantic contexts $\mathfrak{U}$. Since generating realistic open-world values is the ultimate goal of the OCM itself, we instead synthesize weak but structured statistical signals to bootstrap training. To ensure these synthetic signals provide complementary information, they are generated independently of any observed instances in the real datasets $\mathfrak{D}$.

For each semantic context, we formulate a plausible probabilistic data-generating process by prompting an LLM to: (i) construct a Bayesian network, a directed acyclic graph (DAG), encoding conditional dependencies between the features, and (ii) a set of corresponding per-node conditional distributions in the form of conditional probability tables. 
The LLM is prompted once to generate the DAG via a topologically ordered list of features and corresponding edge list, which is 
post-processed to enforce acyclicity.
The conditional probability tables perform a look-up based on parents' values (for discrete parents) or quartiles (for continuous parents).
For discrete nodes, the conditional probability tables directly specifies a pmf over domains. 
For continuous nodes, the conditional probability tables specifies a pdf as Gaussian mean and variance.
% Conditional probabilities for discrete variables are parameterized with generated conditional probability tables, while continuous variables utilize simple parametric distributions sampled from quartiles over the declared ranges for each feature\todojo{continuous description is not clear}. 
Sampling from these graphical models yields our initial pool of weakly instantiated synthetic datasets, $\mathfrak{W}=\{(\tilde{\mathbf{S}}^{(k)}, \tilde{\mathbf{X}}^{(k)})\}_{k=1}^M$. (Please find further details in Supp.~Mat.~$\S$\,\ref{sec:llm})

Since these graphical structures and conditional probabilities are purely hypothetical, we propose to filter the resulting values for both plausibility and informativeness. We train an ensemble of two preliminary OC models on disjoint partitions of the real datasets $\mathfrak{D}=\mathfrak{D}_1\cup\mathfrak{D}_2$, $\mathfrak{D}_1\cap\mathfrak{D}_2=\emptyset$. For each synthetic value $\tilde{x}^{(k,i)}_j$ (the $j^{\mathrm{th}}$ feature of the $i^{\mathrm{th}}$ instance in the $k^{\mathrm{th}}$ dataset), we compute the predictive mean $\bar{x}^{(k,i)}_j$ and standard deviation $\sigma^{(k,i)}_j$ from the two real-data-only OC models, conditioned on the remaining features $\tilde{x}^{(k,i)}_{-j}$. We then calculate the absolute $z$-score:
% $$z^{(k,i)}_j = \left| \frac{\tilde{x}^{(k,i)}_j - \bar{x}^{(k,i)}_j}{\sigma^{(k,i)}_j} \right|$$
\begin{equation}
    z^{(k,i)}_j = \left| \tilde{x}^{(k,i)}_j - \bar{x}^{(k,i)}_j\right| / \sigma^{(k,i)}_j .
    \label{eq:zscore}
\end{equation}
This metric elegantly captures two desiderata: a small numerator indicates the value is plausible (it aligns with the real-data models' expectations), while a large denominator indicates the value is informative (it lies in a region of high predictive uncertainty). Synthetic datasets exhibiting low average feature-level $z$-scores are retained.
%\todo{motivate why train two models on half data here}

\subsubsection{Cross-Model Pseudo-Value Refinement}\label{sec:method:cross}

Although filtered, the weakly generated synthetic values are not data-grounded and could be refined to produce better training signals. To strengthen pseudo-instance values while mitigating confirmation bias, we employ a novel cross-model refinement procedure.\looseness-1 % inspired by co-training \cite{TODO}.

We partition the available data into two disjoint splits and train two independent OCMs: $p'_1$ is trained on $\mathfrak{D}_1 \cup \mathfrak{W}_1$, and $p'_2$ is trained on $\mathfrak{D}_2 \cup \mathfrak{W}_2$. We then perform cross-dataset refinement: $p'_1$ recalculates and refines the synthetic values in $\mathfrak{W}_2$, while $p'_2$ revises those in $\mathfrak{W}_1$, yielding the finalized, high-quality synthetic collections $\mathfrak{W}'_1$ and $\mathfrak{W}'_2$ (see \emph{Step 3}, Fig.~\ref{fig:texr}). That is, refined $\mathfrak{W}'_1$ contains datasets ${\tilde{\mathbf{X}}_1^{' (k)}\sim p'_2(\cdot 
\mid \tilde{\mathbf{S}}_1^{(k)})}$, and vice-versa, $\mathfrak{W}'_2$ has datasets ${\tilde{\mathbf{X}}_2^{' (k)}\sim p'_1(\cdot 
\mid \tilde{\mathbf{S}}_2^{(k)})}$  using the autoregressive, context conditioned OCM generation \eqref{eq:uci_semantic_loss}.  
%\todo{needs more details here, maybe add a algorithm block if not easy to describe in text}

Since partitioning induces distinct exposure to dataset patterns and semantic scenarios, the two models, $p'_1$ and $p'_2$, develop diverse predictive biases. By ensuring each model only revises the pseudo-values of synthetic datasets it did \emph{not} observe during training, this cross-revision mechanism leverages that diversity to stabilize the estimates and prevent models from reinforcing their own local generative errors.

\subsection{Consolidated Model Training}

% In the final phase of TEXR, the refined synthetic collections $\mathfrak{W}'_1$ and $\mathfrak{W}'_2$ are aggregated with the original real-world datasets $\mathfrak{D}$ (\emph{Step 4}, Fig.~\ref{fig:texr}). A single, unified open-world conditional model, $p_{\mathrm{full}}$, is then trained over this consolidated corpus $\mathfrak{D} \cup \mathfrak{W}'_1 \cup \mathfrak{W}'_2$. By optimizing the arbitrary conditional masking objective (Eq.~\ref{eq:ocm_loss}) over this vastly expanded task universe, the resulting model effectively bridges the open-world coverage gap, enabling robust zero-shot and few-shot inference on entirely novel semantic contexts.
In the final stage of TEXR, the refined synthetic collections, $\mathfrak{W}'_1$ and $\mathfrak{W}'2$, are joined with the original real-world datasets, $\mathfrak{D}$, to train a unified open-world conditional model $p{\mathrm{full}}$ on the consolidated corpus $\mathfrak{D} \cup \mathfrak{W}'_1 \cup \mathfrak{W}'_2$ (\emph{Step 4}, Fig.~\ref{fig:texr}). By optimizing the arbitrary conditional masking objective (Eq.~\ref{eq:ocm_loss}) over this expanded task universe, the model mitigates the open-world coverage gap, yielding stronger generalization on zero- and few-shot inference over previously unseen semantic contexts (see below).

\section{Experiments}\label{sec:experiment}

We evaluate OCM cross-tabular models under three deployment regimes: 
\textbf{(I) zero-shot}, \textbf{(II) few-shot}, and \textbf{(III) many-shot inference}. 
Evaluation follows \citet{ye2024crosstablemaskedpretrainingweb} and includes 20 heterogeneous tabular datasets spanning binary classification, multiclass classification, and regression, drawn from UCI, OpenML, and Kaggle.

For synthetic pretraining, we sample 10{,}000 datasets from our generative pipeline, with all generators constrained to the same generative budget. The real-dataset collection, $\mathfrak{D}$, is the 1{,}400 OpenTabs datasets \citep{ye2024crosstablemaskedpretrainingweb}, each paired with natural language dataset- and feature-level descriptions.

\textbf{OC Model Backbones.}\quad
We assess how data expansion affects various open-world conditional modeling cross-tabular backbones: \textsc{ASPIRE} \citep{brahmavar2026universalneurallikelihoodinference}, \textsc{CM2} \citep{ye2024crosstablemaskedpretrainingweb}, \textsc{TabSTAR} \citep{arazi2025tabstar}, and \textsc{TP-BERTa} \citep{yan2024makingpretrainedlanguagemodels}. 
% Unless otherwise stated, all training strategies share the same backbone to isolate the effect of the proposed method rather than model family.

% \todojo{There's a ton a spots were you say `an LLM', which? and why? Check that this is true. feel free to directly just mention models in main text if easy enough.}
\textbf{LLM Backbones.}\quad
Please see Supp.~Mat.~$\S$\,\ref{sec:llm}  for full details of all LLMs used. TEXR relies exclusively on open-weight LLMs. For competing baselines, we follow the originally recommended backbones, additionally evaluating updated variants when available and also used within TEXR.

\subsection{Experimental Baselines}

As no prior work evaluates our full problem formulation, we construct strong baselines by composing state-of-the-art components across the pipeline. To our knowledge, these combinations have not been systematically studied to improve OCM. We factor baselines along two axes: \emph{synthetic task generation} and \emph{pseudo-value refinement}.

\paragraph{Training Protocol}
OC models incorporate synthetic data in a pre-training step before continuing their OCM training on OpenTabs real datasets (see Supp. Mat $\S$\ref{sec:additional_results}).
For each CM backbone, each baseline creates a pipeline by selecting (i) a synthetic task generator, (ii) a pseudo-value refinement strategy.  First, the generator produces a pool of weakly instantiated synthetic datasets (analogously to $\S$\,\ref{sec:method:syn}). Next, the refinement strategy updates these values using the OCM backbone trained on the weakly instantiated synthetic data (analogously to $\S$\,\ref{sec:method:cross}). Finally, the full OC model is trained on the union of refined synthetic and real data and evaluated on held-out test datasets.

\textbf{Synthetic Task Generators.}\quad
% We compare three generators under a fixed dataset budget.  (i) Context Conditioned Bayesian Table Synthesis samples: We generate synthetic tabular datasets by coupling an instruction-tuned LLM with a Bayesian network (BN). First, we sample diverse dataset topics by prompting the LLM to produce a large ``topic universe'' grounded in a broad ``domain universe'' (controlled by the number of domains as 500 and number of topics as 1000). For each selected topic, the LLM outputs a dataset specification consisting of a text dataset description and feature. Next, we construct a per-topic BN over these features: we generate a directed acyclic graph (DAG). We then parameterize conditional probability tables (CPTs) for each node in topological order. Continuous variables are represented via 4-bin quartile mixtures over the declared range, and conditioning on continuous parents is implemented by discretizing parent values into quartiles (Q1–Q4), enabling mixed discrete/continuous parent sets within a unified CPT interface. Finally, we generate rows vi a sampling over the DAG.
We compare the following four generators under a fixed generation budget (further details in Supp.~Mat.~$\S$\,\ref{sec:llm}):
%, treating all of them as alternative weak-instantiation mechanisms for the same semantic contexts. 
\begin{enumerate}[leftmargin=*,topsep=-6pt, noitemsep, wide]
\item[\textbf{(i) Our Task Expansion and Weak-Instantiation:}]
We generate 500 dataset topics ($\S$\,\ref{sec:method:syn}) and sample 20 semantic data contexts per topic (including textual dataset descriptions, textual features descriptions, feature-types). 
For each feature set, we prompt an LLM once to propose a global edge list in topological order. Edges that introduce cycles are discarded to ensure a valid DAG. We then generate 1{,}000 rows via topological sampling over the resulting DAG using the generated conditional probability tables ($\S$\,\ref{sec:method:weak}), producing weak instantiations of the semantic contexts.

\item[\textbf{(ii) StructSynth Weak-Instantiation:}]
As an alternative weak-instantiation baseline, we use \textsc{StructSynth} \citep{liu2025structsynthleveragingllmsstructureaware}. Because \textsc{StructSynth} does not perform task expansion, we apply it to initialize values for semantic data contexts generated in $\S$\,\ref{sec:method:syn}.
\textsc{StructSynth} originally synthesizes instances via LLM-guided prompting conditioned on a few-shot support set and dataset context. In our setting, we remove the support set and prompt the LLM using only our generated dataset contexts.
While both \textsc{StructSynth} and our weak-instantiation ($\S$\,\ref{sec:method:weak}) construct intermediate DAGs, their procedures differ. \textsc{StructSynth} incrementally builds the DAG via iterative breadth-first LLM queries, conditioning on the current graph state, and subsequently autoregressively generates each feature conditioned on its parents through LLM generation. In contrast, TEXR generates the full DAG in a single query and samples values from the generated conditional probability tables ($\S$\,\ref{sec:method:weak}). As in \citep{liu2025structsynthleveragingllmsstructureaware}, all generated values are validated and coerced to the inferred schema for \textsc{StructSynth} (e.g., categorical values mapped to valid categories, continuous values clipped to valid ranges).

\item[\textbf{(iii) CLLM Weak-Instantiation:}]
We consider the Curated LLM (CLLM) pipeline of \citet{seedat2024curatedllmsynergyllms} as an additional weak-instantiation baseline. In its original form, CLLM generates synthetic rows from an in-context few-shot support set followed by a curation stage.
Since CLLM is not designed for zero-shot task expansion, we adapt it by removing the support set and instead providing the same hypothetical dataset contexts generated in $\S$\,\ref{sec:method:syn}. 
Following \citet{seedat2024curatedllmsynergyllms}, we retain the curation mechanism: an XGBoost model scores each synthetic instance using prediction confidence and aleatoric uncertainty, and only the curated subset is retained for downstream training.

\item[\textbf{(iv) TabPFN-based Data Expansion:}] 
TabPFN-style synthetic generation has been shown to improve cross-tabular generalization 
\citep{hollmann2023tabpfntransformersolvessmall, ye2025closerlooktabpfnv2, grinsztajn2026tabpfn25advancingstateart}. 
Within the TEXR formulation, it naturally provides task expansion by generating entirely synthetic datasets without reliance on real data. However, TabPFN operates over abstract feature spaces and does not produce semantic data contexts, instead generating fully instantiated tables with unnamed, uninterpretable columns.
To adapt TabPFN for our OCM setting, we pair its synthetic datasets with \textsc{TabMeta} \citep{anonymous2024tabmeta} to construct corresponding plausible table- and column-level descriptions. Following the TabPFN generation procedure \citep{hollmann2023tabpfntransformersolvessmall}, datasets are sampled from structural causal model (SCM) priors that induce random DAGs with nonlinear structural equations, where values are produced via forward propagation through randomly parameterized neural components. 
Because these priors yield structurally valid but semantically ungrounded features, we generate corresponding semantic data contexts post hoc via \textsc{TabMeta}-style prompting. Multiple LLM judges score candidates for relevance, factuality, and coherence, and we select the highest-scoring annotations via majority voting. This produces synthetic datasets with semantic data contexts while leaving the underlying synthetic feature values unchanged.
\end{enumerate}

% \textbf{TabPFN-based Task Expansion and Weak-Instantiation:} Following the TabPFN paradigm ~\citep{hollmann2023tabpfntransformersolvessmall}, we generate synthetic datasets by sampling random data from structural causal models(SCM). Each dataset is sampled from an SCM prior that is used to generated a random DAG with non-linear structural equations, where values are generated by forward propagation through randomly parameterized neural components. Since these priors produce feature values without semantic meaning (i.e., columns are structurally valid but unnamed/uninterpretable), we attach \textsc{TabMeta} style metadata \citep{anonymous2024tabmeta} by generating a few table-level description and column-level descriptions after synthesis of the synthetic data. Following, multiple LLM judges independently score candidate metadata for factors such as relevance, factuality, and coherence, then we do max-voting of the scores to select the highest-quality metadata annotations. This results in synthetic dataset with dataset and feature information without altering the underlying values in the synthetic tables. This allows a fairer comparison to generators that produce both data and metadata jointly.

\vspace{6pt}
\textbf{Pseudo-value Refinement.}\quad
We similarly compare three pseudo-value refinement techniques to update weakly instantiated synthetic values. A final full model (of corresponding OCM backbones) is trained using all real datasets and synthetic value refined data.
\begin{enumerate}[leftmargin=*,topsep=-6pt, noitemsep, wide]
\item[\textbf{(a) Our Cross-Model Refinement:}] We train two OCM models on disjoint partitions of real and weakly instantiated synthetic data to iteratively refine synthetic values (see $\S$\,\ref{sec:method:cross}, \emph{Step 3} of Fig.~\ref{fig:texr}). Refined values are generated autoregressively, with each feature conditioned on previously generated features within the same instance.\looseness-1
%\textbf{a) Our Cross-Model Refinement},  leverages two OCM models (base on selected OCM backbone) trained on distinct partition of real and weakly-instantiated synthetic data to refine the values of synthetic instances (see $\S$\,\ref{sec:method:cross}, \emph{Step 3} of Fig.~\ref{fig:texr}). Cross-refined values are produced autoregressively, where the next refined instance feature depends on previously generated instance features for that instance.\looseness-1

% \textbf{b) Direct Refinement (DR)}: We compare to a pseudo-value approach that, instead of our cross-refinement approach, directly uses a single OCM model trained on all data to refine the weakly-instantiated values. That is, we train a single preliminary OC model $p'$ trained with all real and weakly-instantiated synthetic data, $\mathfrak{D}\cup\mathfrak{W}$; after training, $p'$ is used to refine values to produce $\mathfrak{W}'$ with instances $\tilde{\mathbf{X}}^{' (k)}\sim p'(\cdot 
% \mid \tilde{\mathbf{S}}^{(k)})$ similarly generated autoregressively. 
\item [\textbf{b) Direct Refinement (DR):}]
As a baseline, we consider a single-model pseudo-value refinement strategy. Instead of cross-refinement, we train one OCM model $p'$ on the full set of real and weakly instantiated synthetic data, $\mathfrak{D} \cup \mathfrak{W}$. After training, $p'$ is used to re-generate (refine) the synthetic values, producing an updated collection $\mathfrak{W}'$ with instances sampled autoregressively as 
$\tilde{\mathbf{X}}^{'(k)} \sim p'(\cdot \mid \tilde{\mathbf{S}}^{(k)})$.

\item[\textbf{c) Cluster-Aware Self-Training (CAST):}]
We additionally compare against CAST \cite{kim2023cast}, a recent tabular pseudo-labeling method that recalibrates prediction confidence using density regularization. As with Direct Refinement, we use a single refinement model (no cross-modeling), but 
employ CAST’s density-aware confidence scores to guide value refinement, using these calibrated confidences when sampling refined values to better propagate instance-level uncertainty to the final model.

\end{enumerate}

\subsection{OCM With Data Expansion}

%We evaluate the baseline models under few-shot/zero-shot and many-shot finetuning setup on unseen held-out 15 classification datasets and 5 regression datasets. We follow the evaluation protocol suggested by the authors procedure for few-shot conditioning like perfoming in-context learning or finetuning on shots. 

\textbf{Few-shot OCM.}\quad
We evaluate data expansion strategies under few-shot inference (${\leq}5$ labeled examples), following each backbone’s protocol. Fig.~\ref{fig:few-shot} reports five-shot F1 across backbones. All expansion methods outperform real-data-only training (\textsc{None}), which only trains on the OpenTabs collection; however, \textsc{TEXR} achieves the largest gains, improving performance by an average of ${\Delta}{=}9.8$ percentage points over real-data-only (\textsc{None}) across models. As shown in Fig.~\ref{fig:zero-few-shot}, \textsc{TEXR} consistently provides the strongest and most stable improvements as the number of shots varies. Similar trends hold for regression (Supp.~Mat.~Fig.~\ref{fig:reg}), where \textsc{TEXR} yields the largest average RMSE reduction.
% \textbf{Few-shot OCM}\quad We begin by studying the effects of various data expansion strategies on OCM when few (${\leq}5$) labeled examples are given at inference time. We follow the respective few-shot learning protocols for each OCM backbone. 
% Fig.~\ref{fig:few-shot} shows the five-shot classification results (F1) across OCM backbones and data expansion strategies.
% As a result of our novel systematic study, we observe that \emph{all data expansion strategies lead to improvement} over the real-data-only baseline (\textsc{None}), which is trained only on the collection of real datasets from OpenTabs.
% However, we see the most substantial improvements by utilizing our proposed \textsc{TEXR}, which yields an average improvement of ${\Delta}{=}9.8$ percentage points (pp) over real-data-only OCM training across backbones. Fig.~\ref{fig:zero-few-shot} shows how OCM is affected as we vary the number of shots. Again, we see that \textsc{TEXR} provides the largest improvement, and most consistent improvement as the number of shots vary.
% We also see similar patterns on regression tasks (see Supp.~Mat.~Fig.~\ref{fig:reg}), where TEXR yields the largest average RMSE improvement.
%of ${\Delta}{=}.154$. Given these low-data OCM tasks, this represents a substantial improvement of open-world conditioning capabilities for a diverse group of models. 
% In particular, our full pipline improves over the baseline by \(\Delta=0.107\) for \textsc{TP-BERTa}, \(\Delta=0.121\) for \textsc{CM2}, and \(\Delta=0.089\) for \textsc{ASPIRE}. 

\begin{figure}[h]
  \centering
  \includegraphics[width=\linewidth]{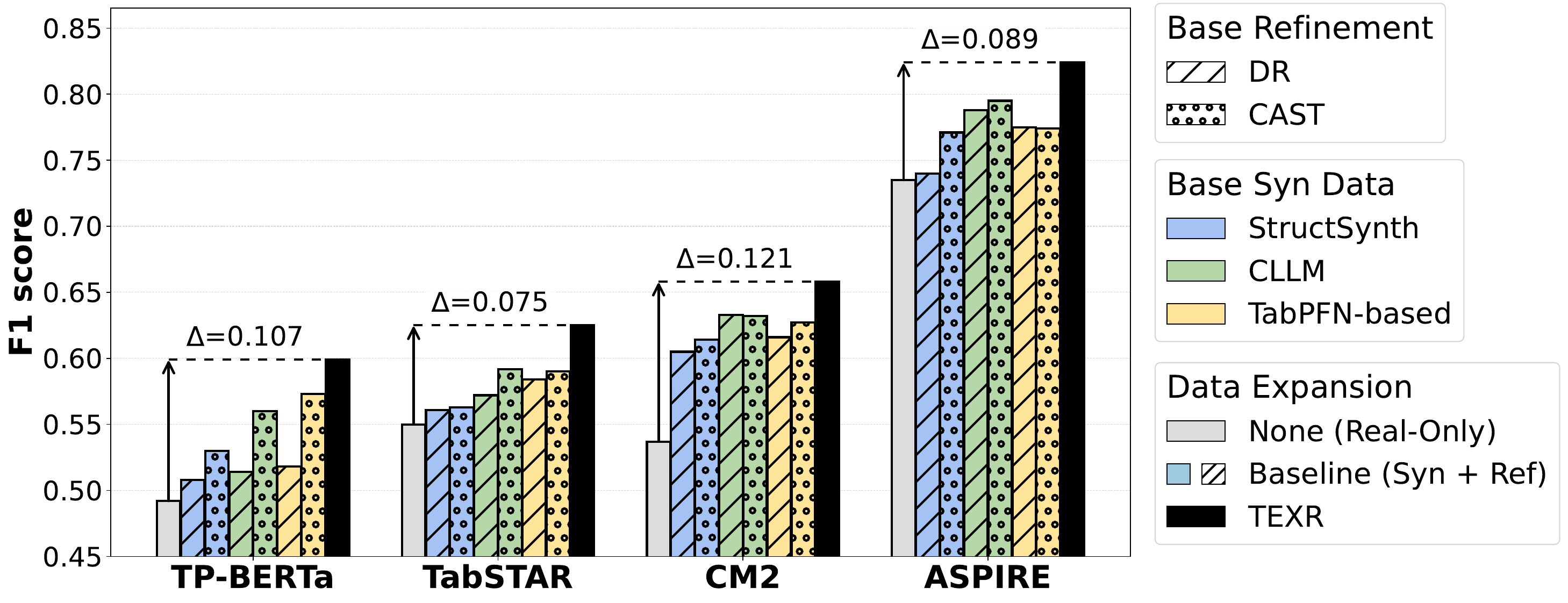}
  \caption{\textbf{Five-shot F1 comparison} across OCM backbones and data expansion strategies: synthetic generator (color), refinement (hatch), real-data-only (gray), and \textsc{TEXR} (black).}
  \label{fig:few-shot}
\end{figure}

\begin{figure}[h]
  \centering
  \includegraphics[width=\linewidth]{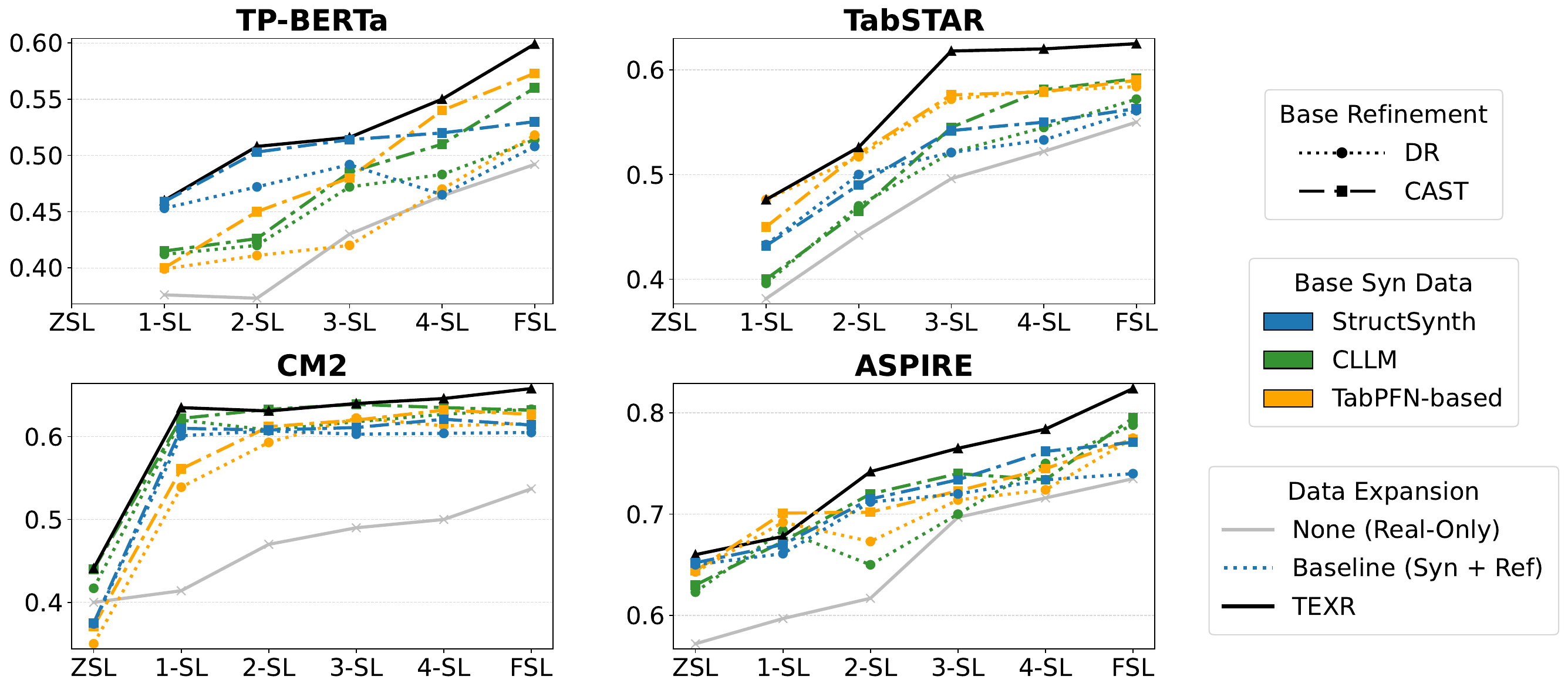}
  \caption{\textbf{Zero-shot (ZSL) to Five-shot (5SL) F1 comparison} across data expansion strategies and OCM backbones. (TP-BERTa and TabSTAR are unable to perform ZSL.)}
  \label{fig:zero-few-shot}
\end{figure}

% \begin{figure}[h]
%   \centering
%   \includegraphics[width=\linewidth]{figures/regression.pdf}
%   \caption{\textbf{Five-shot RMSE comparison} across data expansion strategies over various OCM backbones.}
%   \label{fig:reg}
% \end{figure}

% \textbf{Many-shot OCM}\quad Although our main focus is OCM in data-limited scenarios, we see that TEXR also aides OC models when performing \emph{many-shot} open-world conditioning at inference time. That is, in this setting, we study the effect of data expansion on OCM backbones when they are presented with many examples from the inference-time prediction task (several thousands). Many-shot adaptation is done in accordance to the original protocols of each OCM backbone. We again see substantial improvements over real-data-only training in this many-shot setting. Notably, our proposed TEXR data expansion gives an average improvement of ${\Delta}{=}5.375$ percentage points over real-data-only training for backbones. This a considerable gain, given that in this many-shot setting, the OC models are not as reliant on cross-tabular patterns since they observe many inference-time labeled examples from target tasks.

\textbf{Many-shot OCM.}\quad
Although TEXR is motivated by data-limited regimes, it also improves performance in many-shot conditioning, where backbones observe thousands of labeled inference-time examples under their standard adaptation protocols. Fig.~\ref{fig:finetune} shows consistent gains over real-data-only training, with \textsc{TEXR} improving F1 by an average of ${\Delta}{=}5.38$ percentage points across backbones. This is notable given that, in the many-shot setting, models rely less on cross-tabular transfer due to abundant task-specific supervision. Similar patterns hold for many-shot regression (Supp.~Mat.~Fig.~\ref{fig:reg_finetune}), where \textsc{TEXR} yields the largest average RMSE reduction.

\begin{figure}[h]
  \centering
  \includegraphics[width=\linewidth]{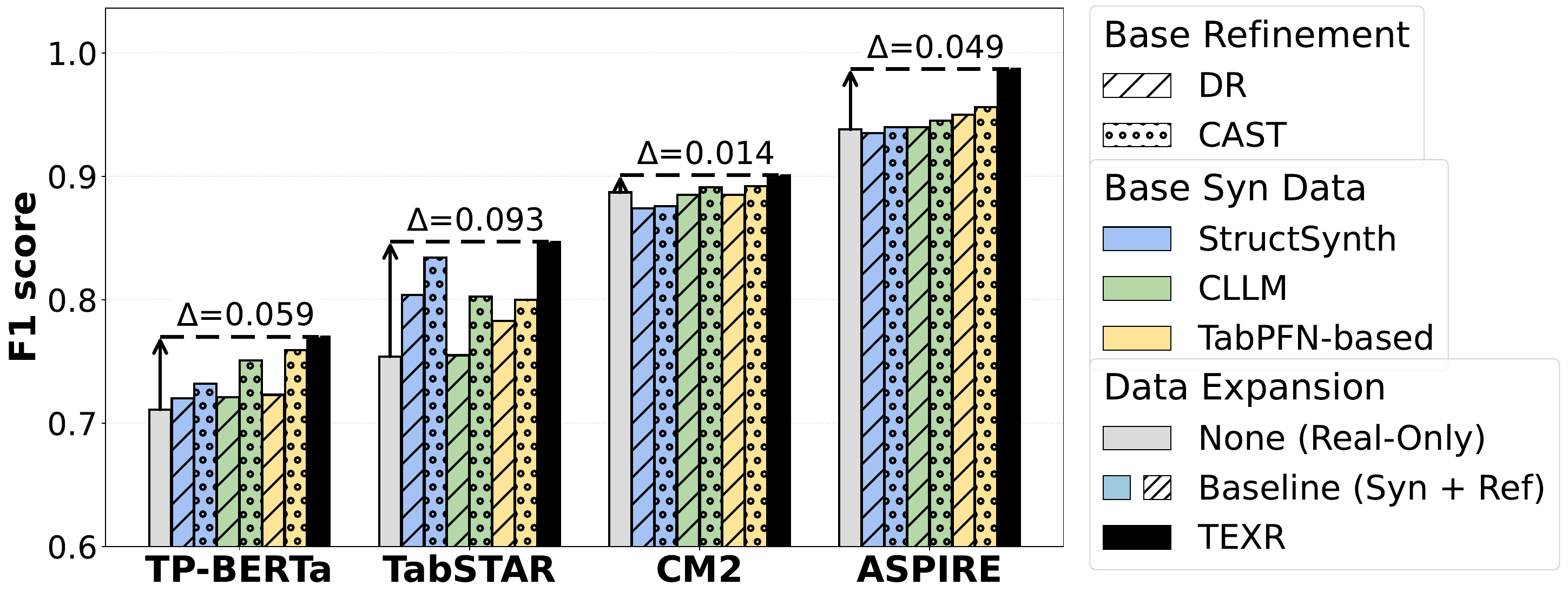}
  \caption{\textbf{Many-shot F1 comparison} across data expansion strategies over various OCM backbones.}
  \label{fig:finetune}
\end{figure}

\subsection{Ablations}
We further analyze \textsc{TEXR} via targeted ablations of its task expansion and pseudo-value refinement components, focusing on the strong \textsc{ASPIRE} backbone.
First, fixing cross-refinement ($\S$\,\ref{sec:method:cross}), we vary the expansion generator among \textsc{StructSynth}, \textsc{CLLM}, and \textsc{TabPFN}-based synthesis. Second, fixing task expansion ($\S$,\ref{sec:method}), we compare refinement strategies: no refinement, direct refinement, and CAST. 
The full pipeline (\textsc{TE+XR}) achieves the largest gain over real-data-only (\textsc{None}). Our task expansion (\textsc{TE}) outperforms alternative generators (\textsc{*+XR}), highlighting the strength of our synthesis and weak-instantiation design. Refinement—especially cross-refinement—further improves over weak instantiation (\textsc{TE+NoR}), confirming its added OCM value.
\textbf{z-score filtering:} removing the filter in Eq.~\eqref{eq:zscore} reduces performance from 0.824 to 0.815, indicating it helps select informative synthetic tasks.

\begin{figure}[h]
  \centering
  \includegraphics[width=\linewidth]{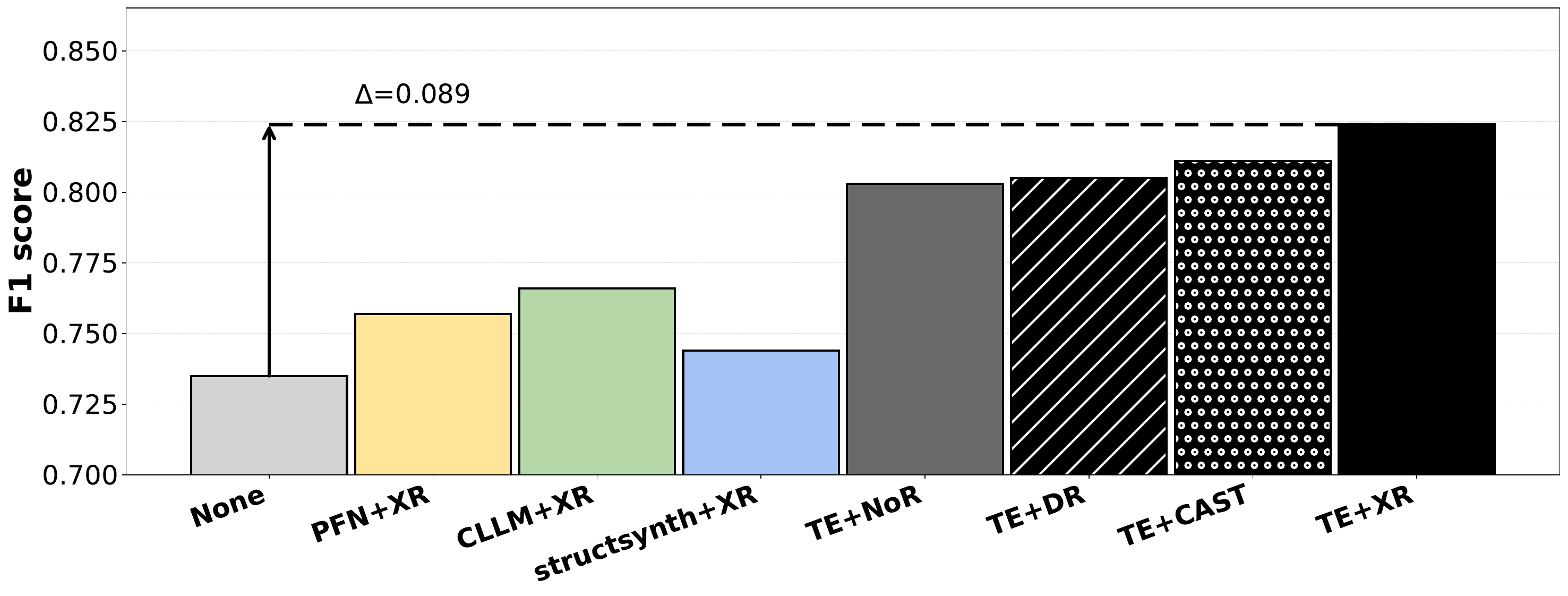}
  \caption{\textbf{Data expansion ablation} on \textsc{ASPIRE}.}
  \label{fig:ablation}
\end{figure}
% \paragraph{Single model training vs single model with synthetic with single model with inter-labelling}
% We ablate the contribution of synthetic data and pseudo-label refinement in our iterative procedure. We compare across pre-refinement stages of the procedure to study the change in performance for the models contributed by the synthetic data. We study the ASPIRE model as it is the best performing model across all evaluation settings. ASPIRE pre-trained on real data yields an average F1 score \textbf{0.735}. Augmenting training with synthetic data boosts the performance to \textbf{0.774} and \textbf{0.752} for the inter-labeling two models. Corresponding to gains of \textbf{+0.039} and \textbf{+0.027} over the real-only baseline indicating addition of synthetic data is beneficial. 

% \paragraph{No refinement}
% To isolate the effect of the refinement stage, we also train \textsc{ASPIRE} on our task expansion and weak-instantiation. This setting we train \textsc{ASPIRE} on this pool of synthetic data and finetune on the real data as specified in ($\S$\,\ref{sec:experiment}). We notice an improvement of \textbf{+0.068}

% \paragraph{LLM zero shot (with no BN)} \todojo{probably lowest priority}.

% % \paragraph{ASPIRE with all combinations of  pseudolabelling and synthetic data}
% \todo{1. some other synthetic data intialization for cross refinement. 2. our task expansion with someother refinement (SSP + CAST) }

\section{Discussion} \label{sec:discussion}

\textbf{A new semantically-grounded tabular prior.}\quad
TabPFN’s synthetic-data paradigm \citep{hollmann2023tabpfntransformersolvessmall, grinsztajn2026tabpfn25advancingstateart} is grounded in a structured generative prior over latent causal graphs, feature marginals, target mechanisms, and noise processes; its inductive bias is explicitly statistical and row-level, shaped by synthetic task distributions that approximate classical machine-learning datasets, yielding strong cross-tabular performance. In contrast, we introduce a \emph{complementary} open-world semantically-grounded prior for generating weakly-instantiated synthetic datasets that encodes broad knowledge of domains, entities, and attribute relationships. As a proper cross tabular instance prior, our LLM-based prior remains largely independent of structured feature–value matrices. This follows from TEXR’s reliance on generalist LLMs, which are pretrained on massive natural-language corpora (e.g., web text, books, code) rather than curated tabular row datasets \citep{brown2020language, touvron2023llama}.
The tight coupling between semantics and plausible weak instance values makes our dataset prior especially adept at improving open-world conditional modeling, as evidenced by the results above.

\textbf{Release of Synthetic Datasets.}\quad 
In addition to our code, we will publicly release\footnote{Upon publication.} a collection of 10{,}000+ \emph{refined} synthetic datasets ($\S$\,\ref{sec:method:cross}) spanning a broad range of topics (Fig.~\ref{fig:word_cloud}, left). 
As supported by our results, these datasets improve multiple OCM backbones, suggesting that they encode useful structural patterns coupled with semantic data contexts that complement large real-world collections such as OpenTabs \citep{ye2024crosstablemaskedpretrainingweb}. A t-SNE visualization (Fig.~\ref{fig:word_cloud}, right) indicates that task expansion increases coverage of regions underrepresented in OpenTabs. We hope this released collection—and the accompanying pipeline—serves as a practical resource for cross-tabular research.

\begin{figure}[h]
  \centering
  \includegraphics[trim={0cm 20.5cm 0cm 20.5cm}, clip, width=.48\linewidth]{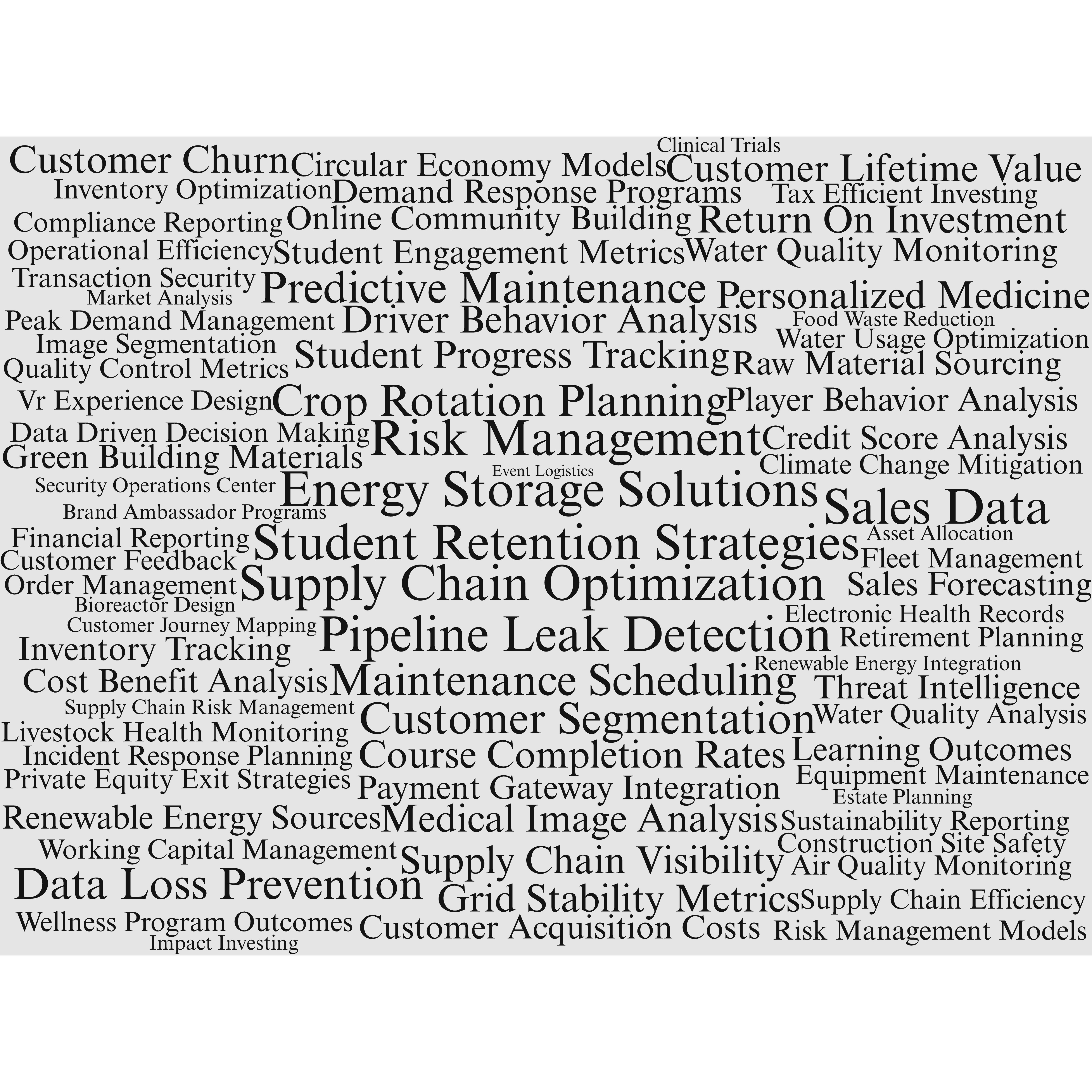}
  \includegraphics[trim={0cm 0cm 0cm 0cm}, clip, width=.44\linewidth]{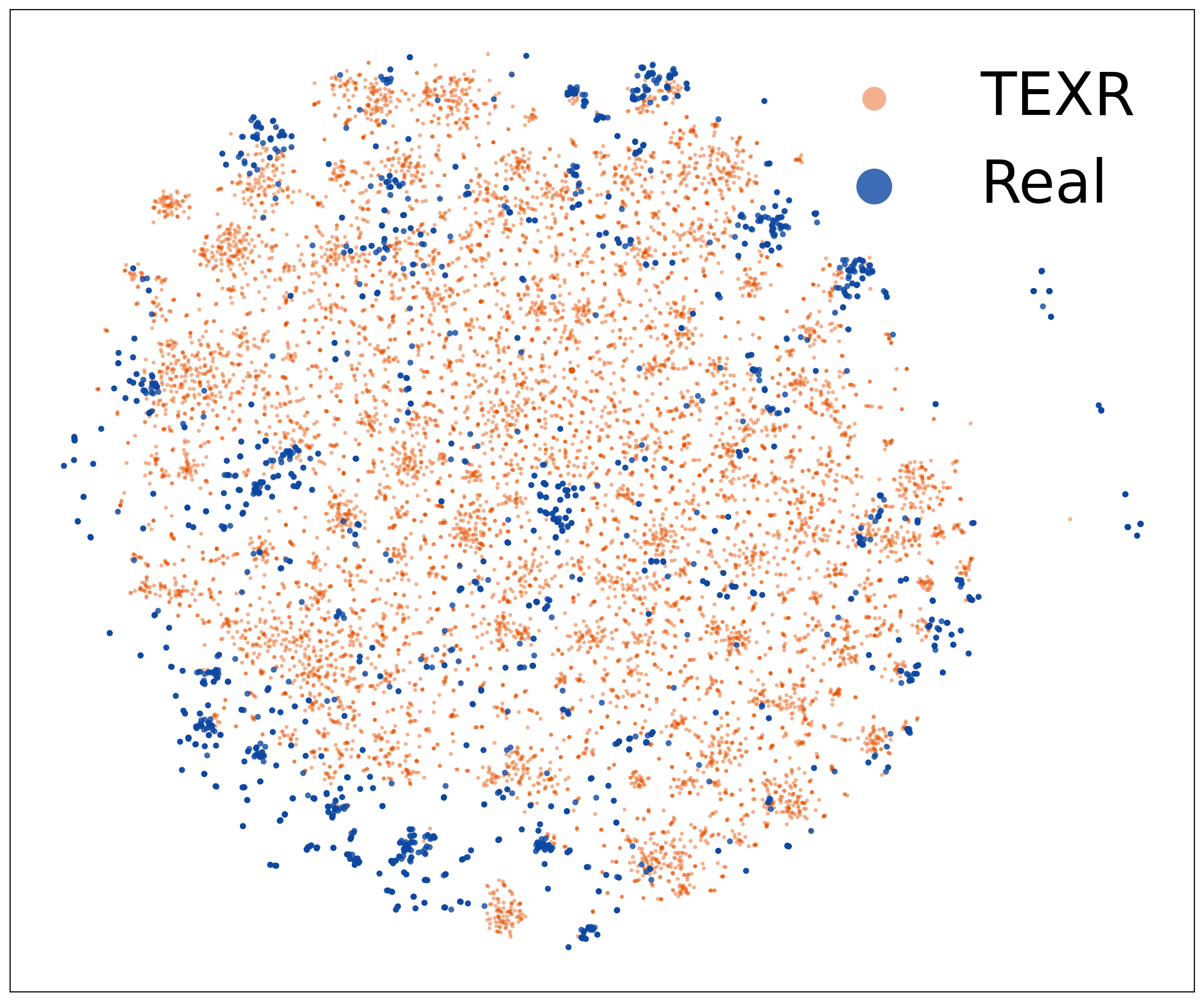}
  \caption{Left, representative \textbf{word cloud} of our synthetic datasets. Right, t-SNE visualization of \textbf{space of synthetic vs.~real data contexts}.}
  \label{fig:word_cloud}
\end{figure}

\section{Conclusion}
In conclusion, TEXR provides a semi-supervised framework for open-world conditional modeling that expands effective task coverage through structured task expansion and cross-model refinement. By treating semantic data contexts as generative inputs, ``unlabeled'' contexts can be transformed into weakly instantiated tabular tasks and refined into useful training signal. This yields a novel, complementary semantic prior to existing statistically grounded synthetic paradigms. Across heterogeneous backbones, TEXR improves zero-/few- and many-shot performance, and our release of 10{,}000+ refined synthetic datasets offers a practical resource for future work on cross-dataset tabular learning. Potentially fruitful future directions include the integration of multiple data expansion strategies jointly, an active-learning search over task expansion, and extensions into multiple modalities.
%\label{sec:discussion}
% References

\bibliography{uai2026-template}

@article{brown2020language,
  title={Language models are few-shot learners},
  author={Brown, Tom and Mann, Benjamin and Ryder, Nick and Subbiah, Melanie and Kaplan, Jared D and Dhariwal, Prafulla and Neelakantan, Arvind and Shyam, Pranav and Sastry, Girish and Askell, Amanda and others},
  journal={Advances in neural information processing systems},
  volume={33},
  pages={1877--1901},
  year={2020}
}

@article{touvron2023llama,
  title={Llama: Open and efficient foundation language models},
  author={Touvron, Hugo and Lavril, Thibaut and Izacard, Gautier and Martinet, Xavier and Lachaux, Marie-Anne and Lacroix, Timoth{\'e}e and Rozi{\`e}re, Baptiste and Goyal, Naman and Hambro, Eric and Azhar, Faisal and others},
  journal={arXiv preprint arXiv:2302.13971},
  year={2023}
}

@article{strauss2021arbitrary,
  title={Arbitrary conditional distributions with energy},
  author={Strauss, Ryan and Oliva, Junier B},
  journal={Advances in Neural Information Processing Systems},
  volume={34},
  pages={752--763},
  year={2021}
}

@inproceedings{lee2013pseudo,
  title={Pseudo-label: The simple and efficient semi-supervised learning method for deep neural networks},
  author={Lee, Dong-Hyun and others},
  booktitle={Workshop on challenges in representation learning, ICML},
  pages={896},
  year={2013},
  organization={Atlanta}
}

@article{kim2023cast,
  title={CAST: Cluster-Aware Self-Training for Tabular Data via Reliable Confidence},
  author={Kim, Minwook and Kim, Juseong and Kim, Ki Beom and Song, Giltae},
  journal={arXiv preprint arXiv:2310.06380},
  year={2023}
}

@inproceedings{wang2024meditab,
  title={MediTab: scaling medical tabular data predictors via data consolidation, enrichment, and refinement},
  author={Wang, Zifeng and Gao, Chufan and Xiao, Cao and Sun, Jimeng},
  booktitle={Proceedings of the Thirty-Third International Joint Conference on Artificial Intelligence},
  pages={6062--6070},
  year={2024}
}

@inproceedings{nam2023stunt,
  title={STUNT: FEW-SHOT TABULAR LEARNING WITH SELF-GENERATED TASKS FROM UNLABELED TABLES},
  author={Nam, Jaehyun and Tack, Jihoon and Lee, Kyungmin and Lee, Hankook and Shin, Jinwoo},
  booktitle={11th International Conference on Learning Representations, ICLR 2023},
  year={2023},
  organization={International Conference on Learning Representations, ICLR}
}

@article{kim2023revisiting,
  title={Revisiting self-training with regularized pseudo-labeling for tabular data},
  author={Kim, Minwook and Kim, Juseong and Song, Giltae},
  journal={arXiv preprint arXiv:2302.14013},
  year={2023}
}

@inproceedings{farahani2022clpl,
  title={CLPL: A Self-supervised Contrastive Learning Pseudo-Labeling Framework for Tabular Data},
  author={Farahani, Abolfazl and Tonekaboni, Navid Hashemi and Rasheed, Khaled and Arabnia, Hamid R},
  booktitle={2022 International Conference on Computational Science and Computational Intelligence (CSCI)},
  pages={7--12},
  year={2022},
  organization={IEEE}
}

@article{yoon2020vime,
  title={Vime: Extending the success of self-and semi-supervised learning to tabular domain},
  author={Yoon, Jinsung and Zhang, Yao and Jordon, James and Van der Schaar, Mihaela},
  journal={Advances in neural information processing systems},
  volume={33},
  pages={11033--11043},
  year={2020}
}

@misc{yan2024makingpretrainedlanguagemodels,
      title={Making Pre-trained Language Models Great on Tabular Prediction}, 
      author={Jiahuan Yan and Bo Zheng and Hongxia Xu and Yiheng Zhu and Danny Z. Chen and Jimeng Sun and Jian Wu and Jintai Chen},
      year={2024},
      eprint={2403.01841},
      archivePrefix={arXiv},
      primaryClass={cs.CL},
      url={https://arxiv.org/abs/2403.01841}, 
}

@misc{ye2024crosstablemaskedpretrainingweb,
      title={Towards Cross-Table Masked Pretraining for Web Data Mining}, 
      author={Chao Ye and Guoshan Lu and Haobo Wang and Liyao Li and Sai Wu and Gang Chen and Junbo Zhao},
      year={2024},
      eprint={2307.04308},
      archivePrefix={arXiv},
      primaryClass={cs.LG},
      url={https://arxiv.org/abs/2307.04308}, 
}

@misc{brahmavar2026universalneurallikelihoodinference,
      title={Towards Universal Neural Likelihood Inference}, 
      author={Shreyas Bhat Brahmavar and Yang Li and Qiyang Liu and Shashank Srivastava and Junier Oliva},
      year={2025},
      eprint={2508.09100},
      archivePrefix={arXiv},
      primaryClass={cs.LG},
      url={https://arxiv.org/abs/2508.09100}, 
}

@misc{wen2024supervisedgenerativenovelparadigm,
      title={From Supervised to Generative: A Novel Paradigm for Tabular Deep Learning with Large Language Models}, 
      author={Xumeng Wen and Han Zhang and Shun Zheng and Wei Xu and Jiang Bian},
      year={2024},
      eprint={2310.07338},
      archivePrefix={arXiv},
      primaryClass={cs.LG},
      url={https://arxiv.org/abs/2310.07338}, 
}

@misc{hegselmann2023tabllmfewshotclassificationtabular,
      title={TabLLM: Few-shot Classification of Tabular Data with Large Language Models}, 
      author={Stefan Hegselmann and Alejandro Buendia and Hunter Lang and Monica Agrawal and Xiaoyi Jiang and David Sontag},
      year={2023},
      eprint={2210.10723},
      archivePrefix={arXiv},
      primaryClass={cs.CL},
      url={https://arxiv.org/abs/2210.10723}, 
}

@misc{han2024largelanguagemodelsautomatically,
      title={Large Language Models Can Automatically Engineer Features for Few-Shot Tabular Learning}, 
      author={Sungwon Han and Jinsung Yoon and Sercan O Arik and Tomas Pfister},
      year={2024},
      eprint={2404.09491},
      archivePrefix={arXiv},
      primaryClass={cs.LG},
      url={https://arxiv.org/abs/2404.09491}, 
}

@misc{brown2020languagemodelsfewshotlearners,
      title={Language Models are Few-Shot Learners}, 
      author={Tom B. Brown and Benjamin Mann and Nick Ryder and Melanie Subbiah and Jared Kaplan and Prafulla Dhariwal and Arvind Neelakantan and Pranav Shyam and Girish Sastry and Amanda Askell and Sandhini Agarwal and Ariel Herbert-Voss and Gretchen Krueger and Tom Henighan and Rewon Child and Aditya Ramesh and Daniel M. Ziegler and Jeffrey Wu and Clemens Winter and Christopher Hesse and Mark Chen and Eric Sigler and Mateusz Litwin and Scott Gray and Benjamin Chess and Jack Clark and Christopher Berner and Sam McCandlish and Alec Radford and Ilya Sutskever and Dario Amodei},
      year={2020},
      eprint={2005.14165},
      archivePrefix={arXiv},
      primaryClass={cs.CL},
      url={https://arxiv.org/abs/2005.14165}, 
}

@misc{wei2022emergentabilitieslargelanguage,
      title={Emergent Abilities of Large Language Models}, 
      author={Jason Wei and Yi Tay and Rishi Bommasani and Colin Raffel and Barret Zoph and Sebastian Borgeaud and Dani Yogatama and Maarten Bosma and Denny Zhou and Donald Metzler and Ed H. Chi and Tatsunori Hashimoto and Oriol Vinyals and Percy Liang and Jeff Dean and William Fedus},
      year={2022},
      eprint={2206.07682},
      archivePrefix={arXiv},
      primaryClass={cs.CL},
      url={https://arxiv.org/abs/2206.07682}, 
}

@misc{hollmann2023tabpfntransformersolvessmall,
      title={TabPFN: A Transformer That Solves Small Tabular Classification Problems in a Second}, 
      author={Noah Hollmann and Samuel Müller and Katharina Eggensperger and Frank Hutter},
      year={2023},
      eprint={2207.01848},
      archivePrefix={arXiv},
      primaryClass={cs.LG},
      url={https://arxiv.org/abs/2207.01848}, 
}

@misc{ye2025closerlooktabpfnv2,
      title={A Closer Look at TabPFN v2: Understanding Its Strengths and Extending Its Capabilities}, 
      author={Han-Jia Ye and Si-Yang Liu and Wei-Lun Chao},
      year={2025},
      eprint={2502.17361},
      archivePrefix={arXiv},
      primaryClass={cs.LG},
      url={https://arxiv.org/abs/2502.17361}, 
}

@misc{grinsztajn2026tabpfn25advancingstateart,
      title={TabPFN-2.5: Advancing the State of the Art in Tabular Foundation Models}, 
      author={Léo Grinsztajn and Klemens Flöge and Oscar Key and Felix Birkel and Philipp Jund and Brendan Roof and Benjamin Jäger and Dominik Safaric and Simone Alessi and Adrian Hayler and Mihir Manium and Rosen Yu and Felix Jablonski and Shi Bin Hoo and Anurag Garg and Jake Robertson and Magnus Bühler and Vladyslav Moroshan and Lennart Purucker and Clara Cornu and Lilly Charlotte Wehrhahn and Alessandro Bonetto and Bernhard Schölkopf and Sauraj Gambhir and Noah Hollmann and Frank Hutter},
      year={2026},
      eprint={2511.08667},
      archivePrefix={arXiv},
      primaryClass={cs.LG},
      url={https://arxiv.org/abs/2511.08667}, 
}

@misc{qu2025tabicltabularfoundationmodel,
      title={TabICL: A Tabular Foundation Model for In-Context Learning on Large Data}, 
      author={Jingang Qu and David Holzmüller and Gaël Varoquaux and Marine Le Morvan},
      year={2025},
      eprint={2502.05564},
      archivePrefix={arXiv},
      primaryClass={cs.LG},
      url={https://arxiv.org/abs/2502.05564}, 
}

@misc{zhang2025limixunleashingstructureddatamodeling,
      title={LimiX: Unleashing Structured-Data Modeling Capability for Generalist Intelligence}, 
      author={Xingxuan Zhang and Gang Ren and Han Yu and Hao Yuan and Hui Wang and Jiansheng Li and Jiayun Wu and Lang Mo and Li Mao and Mingchao Hao and Ningbo Dai and Renzhe Xu and Shuyang Li and Tianyang Zhang and Yue He and Yuanrui Wang and Yunjia Zhang and Zijing Xu and Dongzhe Li and Fang Gao and Hao Zou and Jiandong Liu and Jiashuo Liu and Jiawei Xu and Kaijie Cheng and Kehan Li and Linjun Zhou and Qing Li and Shaohua Fan and Xiaoyu Lin and Xinyan Han and Xuanyue Li and Yan Lu and Yuan Xue and Yuanyuan Jiang and Zimu Wang and Zhenlei Wang and Peng Cui},
      year={2025},
      eprint={2509.03505},
      archivePrefix={arXiv},
      primaryClass={cs.LG},
      url={https://arxiv.org/abs/2509.03505}, 
}

@misc{seedat2024curatedllmsynergyllms,
      title={Curated LLM: Synergy of LLMs and Data Curation for tabular augmentation in low-data regimes}, 
      author={Nabeel Seedat and Nicolas Huynh and Boris van Breugel and Mihaela van der Schaar},
      year={2024},
      eprint={2312.12112},
      archivePrefix={arXiv},
      primaryClass={cs.LG},
      url={https://arxiv.org/abs/2312.12112}, 
}

@misc{liu2025structsynthleveragingllmsstructureaware,
      title={StructSynth: Leveraging LLMs for Structure-Aware Tabular Data Synthesis in Low-Data Regimes}, 
      author={Siyi Liu and Yujia Zheng and Yongqi Zhang},
      year={2025},
      eprint={2508.02601},
      archivePrefix={arXiv},
      primaryClass={cs.LG},
      url={https://arxiv.org/abs/2508.02601}, 
}

@article{arazi2025tabstar,
  title={Tabstar: A foundation tabular model with semantically target-aware representations},
  author={Arazi, Alan and Shapira, Eilam and Reichart, Roi},
  journal={arXiv e-prints},
  pages={arXiv--2505},
  year={2025}
}

@article{anonymous2024tabmeta,
  title   = {Tabmeta: Table metadata generation with {LLM}-curated dataset and {LLM}-judges},
  author  = {Anonymous},
  journal = {Submitted to ACL Rolling Review - June 2024},
  note    = {Under review},
  year    = {2024}
}

@misc{borisov2023languagemodelsrealistictabular,
      title={Language Models are Realistic Tabular Data Generators}, 
      author={Vadim Borisov and Kathrin Seßler and Tobias Leemann and Martin Pawelczyk and Gjergji Kasneci},
      year={2023},
      eprint={2210.06280},
      archivePrefix={arXiv},
      primaryClass={cs.LG},
      url={https://arxiv.org/abs/2210.06280}, 
}

@misc{zhao2025tabulaharnessinglanguagemodels,
      title={TabuLa: Harnessing Language Models for Tabular Data Synthesis}, 
      author={Zilong Zhao and Robert Birke and Lydia Chen},
      year={2025},
      eprint={2310.12746},
      archivePrefix={arXiv},
      primaryClass={cs.LG},
      url={https://arxiv.org/abs/2310.12746}, 
}

@misc{lin2025ctsynfoundationmodelcross,
      title={CTSyn: A Foundation Model for Cross Tabular Data Generation}, 
      author={Xiaofeng Lin and Chenheng Xu and Matthew Yang and Guang Cheng},
      year={2025},
      eprint={2406.04619},
      archivePrefix={arXiv},
      primaryClass={cs.LG},
      url={https://arxiv.org/abs/2406.04619}, 
}

\newpage

\onecolumn

\title{Supplementary Material}

\maketitle

\section{Additional Results}
\label{sec:additional_results}

In this section, we mention additional results and details regarding data pre-processing and pretraining of the models. 

\subsection{Training Data}
We use datasets from OpenTabs \citep{ye2024crosstablemaskedpretrainingweb} for training, validation, and testing. We manually collect dataset descriptions and feature descriptions from UCI ML repository, Kaggle, OpenML, etc. Further, we curate metadata about the dataset by obtaining statistics about the dataset using python functions, for example, collecting potential classes for each target, data type of the feature values etc. Tables which have too few rows and columns, which have unclear and invalid data, are dropped. We identify task as classification or regression based on target value types, followed by min-max normalization of continuous values.

\textbf{ASPIRE}: We use the following specifications while implementing ASPIRE \citep{brahmavar2026universalneurallikelihoodinference}:  Set Transformer-based architecture with model dimension 768, 8 attention heads, and 32 inducing points. The intra-instance Set2Set module and inter-instance aggregator each use 2 ISAB layers. Feature descriptions and categorical values are grounded via a frozen BERT-base-uncased encoder.  Continuous values are encoded using 256 Fourier features. The prediction heads use a 10-component Mixture-of-Gaussians (MoG) for continuous targets and cosine-similarity with a learned temperature for categorical targets. For few-shot inference, ASPIRE conditions on $k$ labeled support examples. 

\textbf{CM2}: We use the following specifications: number of transformer layers 3, attention heads 8, batch size 256, learning rate for finetuning 3e-4, patience is 5. 
For performing few-shot learning, \citet{ye2024crosstablemaskedpretrainingweb} prescribes $k$-shot learning via finetuning the model with $k$ examples sampled from the train set. We observe that this causes high-variance in the performance, therefore we report the mean of all metrics for CM2 few-shot learning experiments.

\textbf{TP-BERTa}: TP-BERTa \citep{yan2024makingpretrainedlanguagemodels} is initialized from a RoBERTa backbone adapted for tabular prediction and is designed for standard supervised tabular prediction via fine-tuning on downstream datasets. Following the authors’ guidance, we train TP-BERTa per dataset with a larger training budget (up to 200 epochs) and early stopping with patience 50 epochs.

\textbf{TabSTAR}: We follow the TabSTAR specifications from \citep{arazi2025tabstar}. We use a pretrained encoder-only text model (e5-small-v2) to process categorical values and a two-layer MLP for the scalar numeric value that projects $1{\to}2d{\to}d$ with ReLU and dropout.  

\subsection{Evaluation Data}

The following table ~\ref{tab:downstream} details information about the downstream datasets we use as test datasets in our evaluation.
\begin{table}[t]
\caption{Table of the downstream datasets in our experiments, along with different information}
\label{tab:downstream}
\begin{center}
\adjustbox{width=\columnwidth,center}%
{%
\scriptsize
\begin{tabular}{|l|c|r|r|r|r|l|}
\hline
Dataset Name & R/C & Samples & Numerical & Categorical & Label Classes & Source \\
\hline
Breast & C & 699 & 9 & 0 & 2 & https://archive.ics.uci.edu/dataset/15/breast+cancer+wisconsin+original \\
Bone & C & 1479 & 2 & 7 & 3 & https://archive.ics.uci.edu/dataset/3/connectionist+bench+choice \\
Diabetes & C & 768 & 8 & 0 & 2 & https://openml.org/d/37 \\
Vehicle & C & 846 & 18 & 0 & 4 & https://archive.ics.uci.edu/dataset/149/statlog+vehicle+silhouettes \\
Satimage & C & 6430 & 36 & 0 & 6 & https://archive.ics.uci.edu/dataset/146/statlog+landsat+satellite \\
Sick & C & 3772 & 7 & 22 & 2 & http://archive.ics.uci.edu/dataset/102/thyroid+disease \\
Analcatdata & C & 797 & 0 & 4 & 6 & https://pages.stern.nyu.edu/jsimonof/AnalCatData/Data/ \\
Pcl & C & 1109 & 21 & 0 & 2 & https://openml.org/d/1068 \\
Adult & C & 48842 & 6 & 8 & 2 & https://archive.ics.uci.edu/dataset/2/adult \\
PhishingWebsites & C & 11055 & 0 & 30 & 2 & https://archive.ics.uci.edu/dataset/327/phishing+websites \\
Cylinder-bands & C & 540 & 18 & 21 & 2 & https://archive.ics.uci.edu/dataset/32/cylinder+bands \\
MiceProtein & C & 1080 & 77 & 4 & 8 & https://archive.ics.uci.edu/dataset/342/mice+protein+expression \\
Car & C & 1728 & 0 & 6 & 4 & https://archive.ics.uci.edu/dataset/19/car+evaluation \\
Segment & C & 2310 & 19 & 0 & 7 & http://archive.ics.uci.edu/dataset/50/image+segmentation \\
Porto-seguro & R & 2000 & 26 & 31 & 2 & https://openml.org/d/44787 \\
Amazon & C & 2000 & 0 & 9 & 2 & https://openml.org/d/44712 \\
Elevators & R & 16599 & 18 & 19 & - & https://openml.org/d/216 \\
Yprop & R & 8885 & 251 & 0 & - & https://openml.org/d/416 \\
Topo & R & 8885 & 266 & 267 & - & https://openml.org/d/422 \\
SAT11 & R & 4400 & 115 & 1 & - & https://www.cs.ubc.ca/labs/algorithms/Projects/SATzilla/ \\
Diamonds & R & 53940 & 6 & 3 & - & https://openml.org/d/42225 \\
House\_sales & R & 21613 & 20 & 1 & - & https://openml.org/d/42731 \\
\hline
\end{tabular}
}
\end{center}
\end{table}

\begin{figure}[h]
  \centering
  \includegraphics[width=\linewidth]{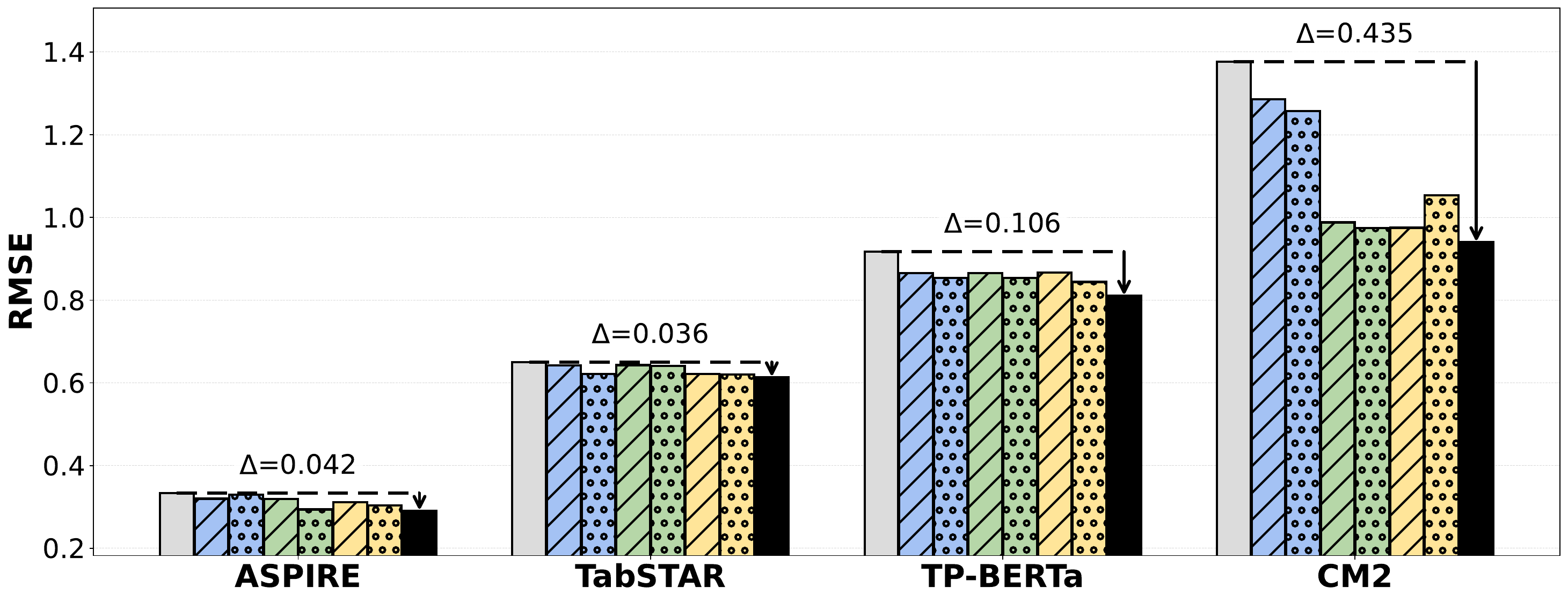}
  \caption{\textbf{Five-shot RMSE comparison} across data expansion strategies over various OCM backbones.}
  \label{fig:reg}
\end{figure}
\begin{figure}[h]
  \centering
  \includegraphics[width=\linewidth]{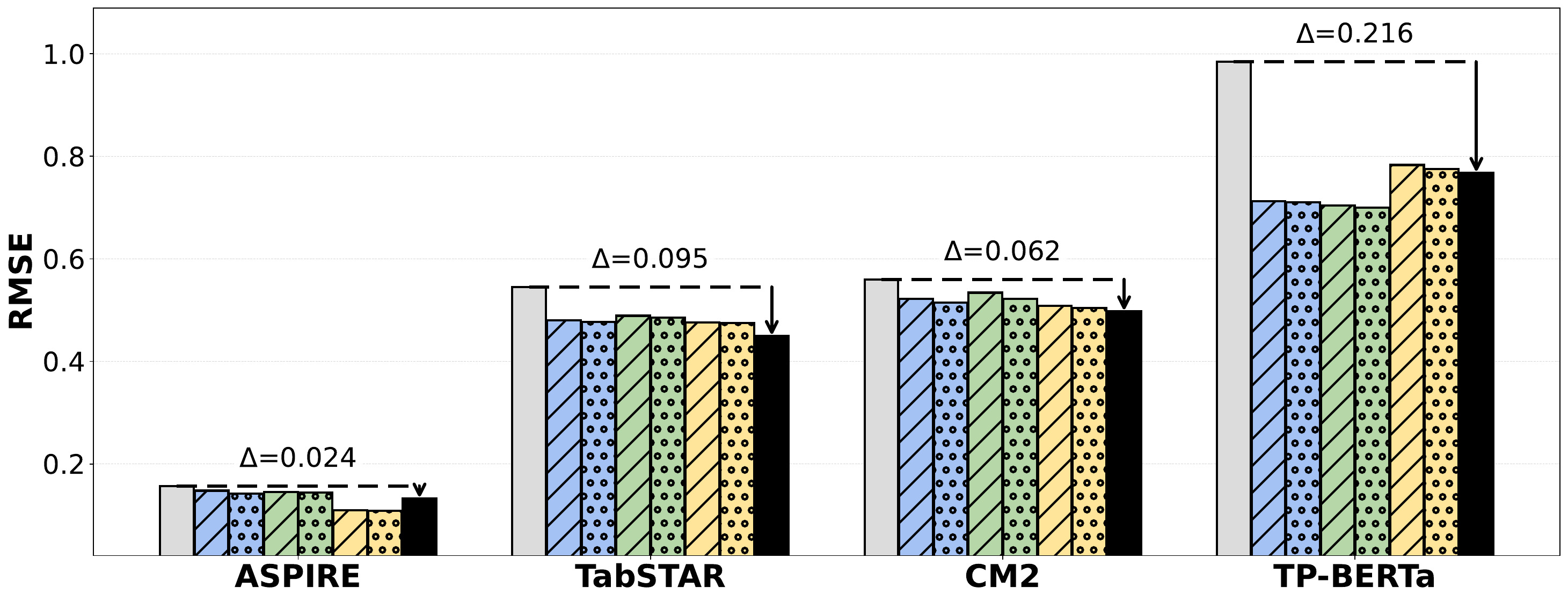}
  \caption{\textbf{Many-shot RMSE comparison} across data expansion strategies over various OCM backbones.}
  \label{fig:reg_finetune}
\end{figure}

\section{Additional details}
\label{sec:llm}

In this section we detail the configuration and prompts to LLMs used in different synthetic data generators like in our \textsc{Our task expansion}, \textsc{CLLM}, \textsc{StructSynth}, \textsc{TabPFN-based}, and \textsc{vanilla LLM}. 
\subsection{\textsc{Our Task Expansion}}
\paragraph{Models.}
We used the following instruction-tuned LLM backbones for synthetic dataset specification generation:
Llama-3.3-70B-Instruct, Meta-Llama-3.1-70B-Instruct, Meta-Llama-3.1-8B-Instruct,
Mistral-7B-Instruct-v0.3, QwQ-32B (and Preview), Qwen3-32B,
Qwen3-30B-A3B-Thinking-2507, and gpt-oss-20b.

Each model independently generated a batch of synthetic datasets.
The final corpus is the union of all datasets produced by these models.
No model ensembling or cross-model mixing was performed. All models were used with the same decoding configuration.

\paragraph{Decoding configuration.}
All models were accessed via HuggingFace \texttt{transformers} using chat-style prompting when supported.
Generation used nucleus sampling with temperature $0.7$, top-$p=0.95$, and a maximum of $2048$ new tokens (input truncated to $4096$ tokens).
We enforce JSON-only outputs and validate responses using a JSON parser; failed generations are retried up to three times.

\paragraph{Prompt design.}
The model is instructed to generate a fixed number of specific and realistic  domains for a given category, returning the result strictly as a JSON array. 

\begin{tcolorbox}[title=System Prompt]
\begin{verbatim}

You MUST respond with ONLY valid JSON. No explanations, no extra text, no 
markdown.
Start immediately with [ or {. End with ] or }. Nothing else.
\end{verbatim}
\end{tcolorbox}

\begin{tcolorbox}[title=Domain Generation Prompt]
\begin{verbatim}
Generate exactly {count} specific business domains/industries for the 
category: {category}

Requirements:
- Return ONLY a JSON array of domain names
- Use specific, descriptive domain names (e.g., "telemedicine" not just 
"healthcare")
- Avoid generic terms, be specific and modern
- Include both traditional and emerging sub-domains
- Make each domain unique and realistic

Format: ["domain1", "domain2", "domain3"]

Generate {count} domains now:
\end{verbatim}
\end{tcolorbox}

\begin{tcolorbox}[title=Topic Generation Prompt]
\ttfamily
Generate exactly \{count\} dataset topics for: \{domain\_1\}, \{domain\_2\}, \{domain\_3\}

Format: ["topic1", "topic2", "topic3"]
Examples: ["customer\_churn", "sales\_data", "inventory\_tracking"]

Generate \{count\} topics:
\end{tcolorbox}

\begin{tcolorbox}[title=Dataset Specification Prompt]
\begin{verbatim}
For topic '{topic}', create dataset specification.

Return this JSON format:
{
  "description": "Dataset description",
  "features": [
    {"name": "age", "type": "continuous", "description": "Age of the subject 
    in years", "range": [18, 80]},
    {"name": "category", "type": "discrete", "description": "Classification 
    category", "categories": ["A", "B", "C"]}
  ]
}

Generate 8-12 realistic features for '{topic}'. Each feature MUST have a 
description explaining what it represents.
\end{verbatim}
\end{tcolorbox}

\begin{tcolorbox}[title=BN Structure Prompt]
\begin{verbatim}
Create edges with for Bayesian Network with features in a topological order: 
{feature_names}

Return edge list as JSON:
[["parent1", "child1"], ["parent2", "child2"]]
\end{verbatim}
\end{tcolorbox}

\begin{tcolorbox}[title=Discrete CPT Prompt]
\begin{verbatim}
Generate probabilities for '{node}' categories: {categories}

Return JSON:
{"cat1": 0.0, "cat2": 0.0}

Make probabilities sum to 1.0:
\end{verbatim}
\end{tcolorbox}

\begin{tcolorbox}[title=Continuous CPT Prompt]
\begin{verbatim}
For feature '{node}' in range [{min}, {max}], parents: {parent_config}.
Return probabilities for 4 quartiles as JSON array [p1, p2, p3, p4].
Probabilities must sum to 1.
\end{verbatim}
\end{tcolorbox}

\subsection{\textsc{TabPFN-BASED}}

Following the procedure in $\S$\,\ref{sec:experiment} to generate data using TabPFN-based data expansion. We leverage the structural causal model (SCM) priors from TabPFN \citep{hollmann2023tabpfntransformersolvessmall} to generate synthetic tabular datasets. Specifically, we sample from two SCM families—\texttt{MLP-SCM} and \texttt{Tree-SCM}—each parameterized with a random number of features (10–50), samples (500–2,000), and seed, producing feature matrices and continous values. For classification tasks, we discretize the values using quantile binning to yield a mix type. Following is the prompt to add metadata and leverage a panel of LLMs-as-judge to filter out metadata. The final metadata is selected by max voting of the judge scores. We use  Meta-Llama-3.1-70B-Instruct, Meta-Llama-3.1-8B-Instruct, Qwen3-32B, Mixtral-8x7B-Instruct-v0.1 models as judges. 
\begin{tcolorbox}[title=Prompt to generate Metadata]
\begin{verbatim}

You are generating realistic metadata for a tabular dataset that should 
sound like it comes from a real-world domain.

IMPORTANT: Generate realistic, human-readable metadata that describes what 
each feature represents in real-world terms.

Requirements:
- Feature names should be MEANINGFUL and domain-appropriate 
(e.g., "patient_age", "temperature", "account_balance", "risk_category", 
"payment_status")
- Feature descriptions should be RICH and explain WHAT
the feature measures in real-world terms:
  * For example continuous: "Patient age at time of hospital admission", 
  "Monthly account balance in customer's primary checking account", 
  "Heart rate measurement during initial assessment"
  * For example categorical: 
  
  "Customer risk classification 
  based on credit history and payment behavior", 
  "Employment status indicating current work situation", 
  "Product category for the transaction"
  
- For categorical features, explain what the categories represent
and how they're used
- Dataset description should tell a compelling story:
what domain, what purpose, what decisions it supports, time period,
data source
- Do NOT include units - just describe what is being measured
- Do NOT mention standardization, normalization, or z-scores in descriptions
- focus on the real-world meaning
- Make descriptions informative and contextual 
(2-3 sentences for dataset, 1-2 sentences per feature)

Return STRICT JSON matching this schema (no markdown, no extra text):
{schema}

EVIDENCE: TABLE
{table}

EVIDENCE: TOP FEATURES FOR TARGET (mutual information)
{top_features_for_target}

EVIDENCE: FEATURE SUMMARIES (truncated)
{features}

EVIDENCE: NEIGHBOR RELATIONSHIPS (correlation-based)
{neighbors}

\end{verbatim}
\end{tcolorbox}

\begin{tcolorbox}[
  title=Prompt to LLM-as-judge,
]
\ttfamily\footnotesize
\begin{verbatim}
You are an impartial judge evaluating metadata quality.

Return STRICT JSON matching this schema (no markdown, no extra text):
{
  "candidate_id": "string",
  "scores": {
    "evidence_consistency": "0-10",
    "specificity": "0-10",
    "helpfulness": "0-10",
    "conciseness": "0-10"
  },
  "total": "0-10",
  "flags": ["string"],
  "fix_suggestions": ["string"]
}

CANDIDATE METADATA (to evaluate)
{candidate_metadata}
Scoring guidelines:
- evidence_consistency: penalize any claim that 
conflicts with numeric types/ranges/stats.
- specificity: reward descriptions that use evidence 
(ranges, distributions, relationships) without overfitting.
- helpfulness: reward metadata that would help a model/user interpret features.
- conciseness: reward clarity without unnecessary verbosity.
\end{verbatim}
\end{tcolorbox}

%SEMANTIC SUGGESTIONS (weak hints, optional)
%{semantic}
% Example GOOD feature descriptions:
% Continuous:
% - "Patient age at time of hospital admission"
% - "Monthly account balance in customer's primary checking account"

% Categorical:
% - "Customer risk classification based on credit history, payment behavior, 
% and debt-to-income ratio"
% - "Employment status indicating current work situation, used to assess income stability for loan approval decisions"

% Dataset description:
% - "Healthcare patient monitoring dataset collected from a regional hospital network between 2020-2023, containing clinical measurements and vital signs used to predict patient readmission risk within 30 days of discharge"

% - "Financial credit risk assessment dataset from a major lending institution, used to evaluate loan default probability based on applicant demographics, credit history, and financial behavior patterns"
\subsection{\textsc{CLLM}}
As prescribed by the authors, we use Mixtral-8x7B-Instruct-v0.1 as the LLM here to prompt and generate synthetic data with temperature $T{=}0.9$. For the curation step we using XGBoost model with $\textit{n\_estimators}{=}100$ trained on a sample of OpenTabs dataset and used to characterise each synthetic example by its confidence and aleatoric uncertainty.
\begin{tcolorbox}[title=Prompt]
\begin{verbatim}
You are a synthetic data generator. 
Your goal is to produce as diverse samples as possible.

Leverage your knowledge of \textit{[dataset description]}
to generate 1000 realistic but diverse samples.

\texttt{\{format\_instructions\}}
The output should be a markdown code snippet 
formatted in the following schema.
```json
{
    "feature\_1": string 
    "feature\_2": string
    ...
    "target": string 
}
\end{verbatim}
\end{tcolorbox}

\subsection{\textsc{StructSynth}}

We retain the original two-stage StructSynth algorithm unchanged.
To enable zero-shot operation, we introduce an upstream zero-shot seed generation phase that constructs synthetic dataset specifications, seed rows, and summary statistics, which are then supplied to StructSynth in zero-shot mode.

\subsubsection{\textsc{0-Shot Seed Generation Overhead}}

The upstream zero-shot generation phase follows the same task-expansion paradigm introduced earlier, extending it from specification generation to full dataset and summary-statistics construction.

\paragraph{Domain and Topic Universe.}
We first generate a universe of domains and dataset topics via LLM prompting.
Each dataset is defined by:
(i) domain,
(ii) topic,
(iii) a structured dataset specification containing 6--10 features with mixed types.

\paragraph{Seed Dataset Construction.}
For each specification, the LLM generates $n=20$ seed rows subject to strict schema constraints.
All values are validated and normalized.

\paragraph{Summary Statistics Generation.}
We additionally prompt the LLM to produce realistic population-level summary statistics, including:
\begin{itemize}
    \item categorical probability distributions,
    \item continuous mean/std/range parameters,
    \item pairwise association scores in $[0,1]$.
\end{itemize}

If summary statistics are incomplete or invalid, deterministic fallbacks are applied using empirical estimates from the seed data.

\subsubsection{\textsc{StructSynth in Zero-shot Mode.}}
The generated summary statistics are passed to StructSynth in zero-shot mode.
In particular, statistical association scores used during dependency discovery are derived from the provided summary statistics rather than computed directly from data.
In this configuration:
\begin{itemize}
    \item schema inference uses provided summary statistics,
    \item association scores are replaced by summary-based values,
    \item no real dataset beyond the synthetic seed is required.
\end{itemize}

StructSynth then generates $n_{\text{synth}}=200$ structured samples per dataset using the original two-stage algorithm.

\paragraph{Configuration.}
All stages use \texttt{gpt-4o-mini} as the backbone model.
Seed generation uses temperature $0.7$.
StructSynth uses temperature $0.9$.
Few-shot size $k=20$.
Row-batch size for structured generation is 1 unless otherwise specified.
All outputs are JSON-validated with up to three retries per call.

\subsection{\textsc{Vanilla LLM}}

\paragraph{Model Configuration.}
We use \texttt{Qwen/Qwen3-4B-Instruct-2507} served via vLLM with
temperature $0.7$, top-$p=0.95$, maximum $512$ new tokens per call,
and maximum context length $1024$. Rows are generated in batches of $8$.
We generate datasets with $100$ samples per dataset under a fixed wall-clock budget of $48$ hours on a single 16GB GPU, ensuring computational parity with other methods. Under this constraint, the vanilla LLM baseline produced a total of $227$ distinct datasets.
The domain universe size is $150$ and topic universe size is $1000$.

\paragraph{Generation Procedure.}
The vanilla baseline follows the same domain, topic, and dataset specification
prompt templates described above for \textsc{Our Task Expansion}
and the 0-shot upstream stage of \textsc{StructSynth}.
However, no structural modeling (e.g., dependency graphs, CPTs, or summary-based
association scores) is applied. Synthetic rows are generated directly from the
LLM conditioned only on the dataset specification and feature constraints.

\end{document}